\documentclass[journal]{IEEEtran}
\usepackage{filecontents}
\usepackage{amsmath,amsthm,dsfont,tabularx,scalefnt,comment}
\usepackage[cmintegrals]{newtxmath}
\usepackage[nocompress]{cite}
\usepackage{subcaption}
\usepackage{caption}
\usepackage{bm}
\usepackage{color,soul}
\usepackage{multirow}
\usepackage{siunitx}
\usepackage{setspace,adjustbox}

\newtheorem{theorem}{Theorem}[section]

\usepackage{booktabs}
\usepackage{algorithm}
\usepackage{algorithmic}

\usepackage{amsmath,amsthm}
\usepackage{mathtools}

\newtheorem{lemma}{Lemma}
\newtheorem{proposition}{Proposition}
\newtheorem{assumption}{Assumption}

\begin{document}

\bibliographystyle{unsrt}

\title{Personalized Federated Learning by Energy-Efficient UAV Communications}
\author{Shiqian~Guo,~
        Jianqing~Liu,~\IEEEmembership{Member,~IEEE,}
        Beatriz Lorenzo,~\IEEEmembership{Senior Member,~IEEE,}
\thanks{Corresponding author: Jianqing~Liu}
\thanks{This work was supported in part by National Science Foundation under grants ECCS-2312738 and CNS-2225427.}
\thanks{Shiqian Guo,  Jianqing Liu are the
	Department of Computer Science, North Carolina State University, Raleigh, NC 27606 USA (e-mail: sguo26@ncsu.edu; jliu96@ncsu.edu). Beatriz Lorenzo is with the Department of Electrical and Computer Engineering, University of Massachusetts, Amherst, MA 01003 USA (e-mail:  blorenzo@umass.edu).}}

\markboth{IEEE TRANSACTIONS ON WIRELESS COMMUNICATIONS,~Vol.~X, No.~X, X~X}%
{Shiqian \MakeLowercase{\textit{et al.}}: Personalized Federated Learning by Energy-Efficient UAV Communications}
\maketitle


\begin{abstract}
Federated learning (FL) is an effective paradigm for enhancing the learning capability of edge devices while preserving data privacy. In geographically dispersed FL systems, such as sensor networks in remote areas, unmanned aerial vehicles (UAVs) can flexibly establish high-quality communication links to support parameter exchange. However, device heterogeneity and the limited battery capacity of UAVs pose significant challenges. Specifically, data heterogeneity slows convergence, while scheduling all devices for global collaboration incurs excessive communication and energy costs. To overcome these challenges, we adopt a strict separation between a globally shared backbone and permanently local personalization heads, thereby mitigating the impact of data heterogeneity. Furthermore, we propose a gradient-based scheduling strategy that jointly considers energy efficiency and learning performance. In each communication round, the backbone is updated only by the top-$\alpha$ devices ranked by gradient $\ell_{2}$-norm, ensuring that optimization focuses on the most informative updates. Simulation results demonstrate that the proposed scheme achieves higher learning accuracy than state-of-the-art approaches while significantly reducing UAV energy consumption.
\end{abstract}


\IEEEpeerreviewmaketitle

\section{Introduction}

\IEEEPARstart{F}{L} is a distributed machine learning architecture designed to train models collaboratively across multiple devices while keeping data localized. It has received a wide range of attention for its promises of producing high-quality models and maintaining data privacy~\cite{9159929}. In the model training phase, devices in an FL system use their local dataset to compute model parameters and communicate them to a server for global aggregation and model update. Subsequently, the server sends the global model back to the devices for further training. The iteration between devices and server continues till the model converges. The distributed and collaborative nature of FL makes it a natural fit on the network edge, where sensors, robots, drones, and other communication devices can be coordinated to train an FL engine. Among recently proposed FL use cases~\cite{9453811,9709639,9606829}, researchers have been specifically advocating that the use of UAVs as FL servers can better support edge devices in inaccessible places (e.g., mountains or battlefields), because UAVs can offer high-quality, agile, and reconfigurable communication channels to/from the devices.

Nonetheless, UAV-aided FL systems face several unique challenges stemming from the stringent energy constraints of UAVs and the severe data heterogeneity across edge devices. To address these issues, recent studies have explored energy-efficient and personalized FL schemes. For mitigating data heterogeneity, Jiang et al.~\cite{jiang2024hete} propose a hybrid strategy that combines local and global updates, while Arivazhagan et al.~\cite{arivazhagan2019federated} introduce a split learning architecture in which a globally shared backbone is paired with permanently local personalization heads. Although effective in improving personalization, these approaches require full-device participation in every communication round, resulting in substantial communication overhead and excessive energy consumption in large-scale UAV-assisted networks.

Frequent transmission of high-dimensional model parameters between UAVs and edge devices can rapidly deplete UAV batteries, prolong the FL training process, and ultimately degrade global model accuracy. Moreover, aggregating updates from less informative devices may dilute the contributions of highly informative ones, leading to slower convergence. To improve training efficiency, Jiang et al. and Guo et al.~\cite{jiang2024over,guo2023energy} develop gradient-based, heterogeneity-aware, and energy-conscious device scheduling mechanisms that selectively activate devices based on their optimization relevance. However, these methods still rely on FedAvg-style aggregation of full model parameters, which can be suboptimal under severe data heterogeneity.

To overcome these limitations, this work proposes a gradient-based scheduling framework that jointly optimizes energy efficiency and learning performance. The proposed method strictly separates the globally shared backbone from permanently local personalization heads, thereby effectively alleviating the adverse impact of heterogeneous data distributions. In each communication round, only the top-$\alpha$ devices ranked by the gradient $\ell_{2}$-norm are selected to participate in backbone updates. This design prioritizes the most informative optimization signals while significantly reducing redundant communication and energy consumption. The main contributions of this work are summarized as follows:\par

\begin{itemize}
\item We investigate a UAV-aided FL system under data heterogeneity and adopt a personalized learning framework that separates a globally shared backbone from permanently local personalization heads to mitigate performance degradation caused by non-IID data. To improve learning efficiency, we introduce a gradient-based scheduling strategy that selects only the top-$\alpha$ devices ranked by gradient $\ell_2$-norm in each round, thereby reducing unnecessary participation and UAV energy consumption. To further optimize energy usage, we incorporate a communication-aware device and UAV selection scheme that jointly considers learning efficiency and communication energy consumption.

\item We provide a theoretical convergence guarantee for the proposed personalized FL framework with top-$\alpha$ gradient-based device scheduling. The analysis shows that scheduling all devices in each round is not optimal, as it may increase personalization-induced bias without improving convergence. In contrast, selecting only the top-$\alpha$ devices balances selection bias and personalization bias while reducing communication overhead and energy consumption. Experimental results further confirm that the theoretical trends agree with practical training behavior.

\item Extensive simulations further validate the proposed framework, demonstrating faster convergence and lower energy consumption than existing UAV-aided FL approaches. The proposed scheme also remains robust under varying degrees of data heterogeneity and different numbers of UAVs, consistently outperforming representative UAV-assisted FL baselines.
\end{itemize}

The remainder of this paper is organized as follows. Section II reviews related work. Section III introduces the system model and analyzes the energy consumption of the UAV-aided FL system. Section IV presents the problem formulation and solution algorithm. Section V provides the simulation results, and Section VI concludes the paper.

\section{Related Works}
\subsection{UAV-aided Federated Learning}
Recent studies have extensively explored the integration of UAVs into FL systems to enhance coverage flexibility, communication efficiency, and learning performance \cite{wang2024energy,asheralieva2025effective,chien2025density,wang2024decentralized}.
In particular, Tang et al.~\cite{tang2025integrated} propose an integrated sensing–computation–communication framework for UAV-assisted federated edge learning, emphasizing the necessity of cross-layer co-design in dynamic aerial networks.

A growing body of work leverages reinforcement learning (RL) and federated reinforcement learning to address the high mobility and dynamic environments of UAV networks. Li et al.~\cite{10423758} train local RL models on individual UAVs and perform federated aggregation on a leader UAV to reduce energy consumption and task completion time in traffic offloading scenarios.Zhang et al.~\cite{zhang2025age} investigate age-of-information minimization in UAV-enabled IoT networks via federated RL, while Tarekegn et al.~\cite{tarekegn2025trajectory} adopt a federated multi-agent deep RL framework to jointly optimize UAV trajectory control and communication fairness in multi-UAV systems.

Beyond learning-based control, several works focus on scalable FL architectures for large-scale UAV networks. Zhagypar et al.~\cite{zhagypar2025uav} study UAV-assisted unbiased hierarchical FL and provide rigorous performance and convergence analysis. Cui et al.~\cite{10314794} develop an asynchronous FL framework for UAV swarms operating over wide geographic regions. To address heterogeneity among UAVs, He et al.~\cite{10050796} propose a cluster-based FL architecture that mitigates learning degradation caused by diverse UAV capabilities. Furthermore, Tang et al.~\cite{10190740} and Yang et al.~\cite{9453811} investigate multi-UAV server architectures, where multiple UAVs act as aggregation nodes and cooperatively optimize model exchanges with ground devices to improve learning performance in large-scale deployments. However, in these approaches, UAVs typically do not directly exchange learned models with one another, which restricts inter-UAV collaboration and limits system scalability.

To overcome this limitation, Saif et al.~\cite{10550002} consider UAV-aided FL with energy-efficient UAV-to-UAV (U2U) communications, enabling global model exchanges among UAVs after local aggregation of devices within each UAV’s coverage.
Despite these advances, most existing UAV-aided FL systems still rely on the conventional FedAvg algorithm \cite{mcmahan2017communication}, which is known to exhibit slow convergence and performance degradation under data heterogeneity \cite{9743558}.

\subsection{Personalized Federated Learning}
Personalized FL has demonstrated its capability to address challenges across diverse application domains \cite{10384325,10159650,10629203,10334003}. In \cite{10159650}, Wang et al. proposed a semi-supervised FL framework for wireless communication systems to enhance learning performance under label deficiency and device heterogeneity. Zhou et al. in \cite{10384325} developed a personalized multi-modal fusion network within an FL system to address challenges arising from multi-modal data in metaverse-enabled smart applications. Several works have investigated personalized FL in UAV-aided networks. In \cite{10629203}, Wang et al. proposed that each UAV independently trained an RL model to maximize energy efficiency in task offloading; while a cloud-based FL framework then aggregated a global model adapted to heterogeneous RL tasks. Similarly, Xu et al.\cite{10334003} proposed a federated RL approach with ensemble distillation and regularization mechanisms to optimize throughput in heterogeneous UAV systems.

\subsection{Energy Efficiency for UAV-aided Personalized FL}
Energy efficiency has emerged as a central design objective in UAV-aided FL, due to the stringent energy constraints of both aerial platforms and edge devices. 
Recent advances in model compression and quantization-aware collaborative inference have been shown to significantly reduce communication and computation overhead for edge devices, which is particularly beneficial for energy-constrained federated learning scenarios \cite{lyu2026quantization}.
Several works focus on UAV trajectory and scheduling optimization to improve energy efficiency. Fu et al.\cite{fu2025energy} formulated an energy-efficient UAV-assisted FL framework that jointly optimizes UAV trajectories, device scheduling, and communication resources. Complementarily, Wu et al.\cite{wu2025joint} studied joint UAV deployment and edge association, showing that optimized UAV placement and user association can balance communication energy consumption and learning accuracy, particularly in large-scale FL scenarios.
Beyond trajectory design, learning paradigm adaptations have been explored to further reduce energy costs. Zhang et al.\cite{zhang2025energy} investigated federated distillation over UAV-assisted wireless networks, where lightweight knowledge transfer replaces full model exchange, achieving substantial energy savings and latency reduction compared to conventional FL.
Other studies emphasize advanced multiple access and resource allocation strategies. Consul et al.\cite{consul2025energy} considered NOMA-based UAV-assisted MEC networks and proposed energy-efficient distributed learning schemes. Similarly, Li et al.\cite{li2025distributed} addressed resource-efficiency optimization in UAV-enabled IoT networks via distributed algorithms that coordinate communication and computation resources among UAVs and edge devices. Jiang et al.\cite{jiang2025joint} examined hierarchical federated learning in UAV-aided marine IoT systems, jointly optimizing communication and computing resources to minimize energy consumption across multiple FL tiers. 

\section{UAV-aided Personalized Federated Learning}
\subsection{System Model}

\begin{figure}[htpb]
	\centering
	\subfloat[System setup for the UAV-aided personalized FL system.]{
		\includegraphics[width=3in]{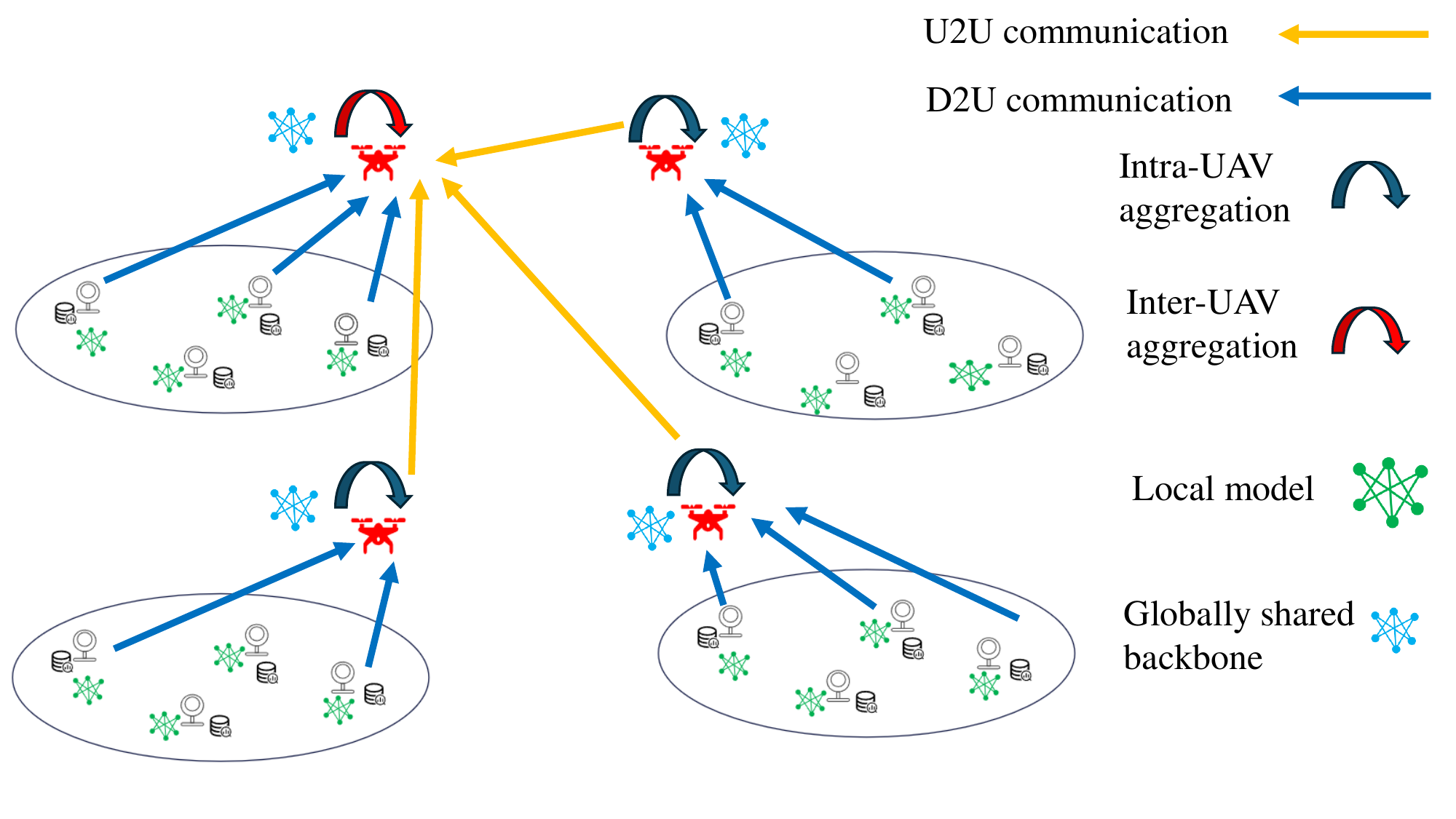}}\\
        	\subfloat[Local model architecture]{
		\includegraphics[width=2in]{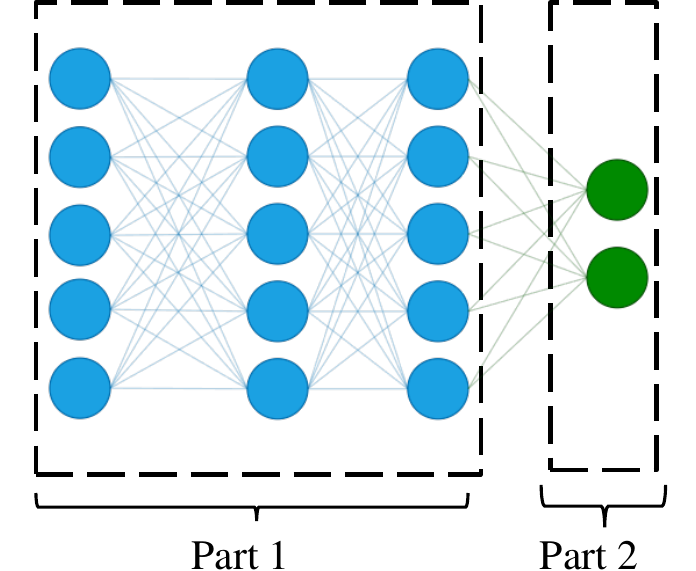}} 
	\caption{A UAV-aided FL system with data-heterogeneous devices, where only the base part of local models is exchanged via D2U (device-to-UAV) and U2U communications. The system bandwidth is allocated using Orthogonal Frequency-Division Multiplexing (OFDM) for model transmission, while the personalized part is permanently trained on local data. In each round, only the base parts of the top-\(\alpha\) devices—selected according to the \(\ell_2\)-norm of their base-part gradients—are uploaded for global aggregation. Part 1 in (b) denotes the globally shared backbone, while Part 2 corresponds to the permanently local personalization heads.}
    \label{fig:uavPFL-system}
\end{figure} 

As illustrated in Fig.~\ref{fig:uavPFL-system}, we consider a UAV-aided personalized FL system consisting of \( U \) rotary-wing UAVs and \( K \) edge devices, uniformly distributed within each UAV’s communication area.  Each UAV serves a local sub-area, and inter-UAV aggregation enables collaboration across the entire network. The UAVs are denoted by \( \mathcal{U} = \{1, \dots, U\} \) and provide a viable solution in remote areas where terrestrial base stations are unavailable or impractical to deploy. Due to the high energy consumption of UAV flight propulsion and the limited onboard energy budget, we assume that UAVs hover at a fixed altitude \( L_u \), which is set as a constant for safety reasons \cite{8663615}. Accordingly, the UAV locations are predetermined and fixed. Each edge device, denoted by \( \mathcal{K} = \{1, \dots, K\} \), has moderate computing and storage resources to perform local training tasks. 

 Following the personalized method in\cite{arivazhagan2019federated}, all edge devices perform local training using a model partitioned into a \emph{base part} (shared parameters) and a \emph{personalized part} (device-specific parameters).
At communication round $t$, device $k$ computes stochastic gradients at $(\boldsymbol{\theta}(t),{\boldsymbol{\phi}}_k(t))$:
\begin{align}
\begin{split}
&{\boldsymbol{g}}_k(t,\boldsymbol{\theta})=\nabla_{\boldsymbol{\theta}} F_k({\boldsymbol{\theta}}(t),{\boldsymbol{\phi}}_k(t))+{\mathbf{\zeta}}_k(t), \\
&{\mathbf{g}}_k(t,{\boldsymbol{\phi}})=\nabla_{{\boldsymbol{\phi}}_k}F_k(\boldsymbol{\theta}(t),{\boldsymbol{\phi}}_k(t))+\upsilon_k(t).
\end{split}
\label{eq:device:parameter:update}
\end{align}

where $\mathcal{I}_t$ denotes the filtration (history) information available up to round $t$, ${\mathbf{\zeta}}_k(t)$ is stochastic noise in the $\boldsymbol{\theta}$-gradient at device $k$, $\mathbb{E}[{\mathbf{\zeta}}_k(t)\mid\mathcal{I}_t]=0$, $\mathbb{E}\|{\mathbf{\zeta}}_k(t)\|^2\le\sigma_{\boldsymbol{\theta}}^2$. $\upsilon_k(t)$ represents stochastic noise in the ${\boldsymbol{\phi}}_k$-gradient at device $k$, $\mathbb{E}[\upsilon_k(t)\mid\mathcal{I}_t]=0$, $\mathbb{E}\|\upsilon_k(t)\|^2\le\sigma_{\boldsymbol{\phi}}^2$.

 The global objective is bilevel:
\begin{equation}
    \min_{\boldsymbol{\theta}} 
    F(\boldsymbol{\theta}):=\frac{1}{K}\sum_{k=1}^K F_k\big(\boldsymbol{\theta},{\boldsymbol{\phi}}_k^\star(\boldsymbol{\theta})\big),
    \label{eq:outer-problem}
\end{equation}

where ${\boldsymbol{\phi}}_k^\star(\boldsymbol{\theta})=\arg\min_{{\boldsymbol{\phi}}_k}F_k(\boldsymbol{\theta},{\boldsymbol{\phi}}_k)$. \(\boldsymbol{\theta} \in \mathbb{R}^{M_{\boldsymbol{\theta}}}\)denotes the model parameter vector of dimension \(M_{\boldsymbol{\theta}}\). The operator \(\min\) denotes the optimization over the model parameter to minimize the average loss across all devices.

Upon completing local training, each edge device computes the \(\ell_2\)-norm of the gradient of its base part. Each UAV first collects the gradient norms from the edge devices within its communication range via D2U communications. The norms are then forwarded among UAVs via U2U communications to a designated UAV that collects all gradient norms. Since the gradient norms are significantly smaller in size compared to the local model parameters, the corresponding communication time and energy overhead are assumed to be neglected.  

Then, based on the collected gradient norms, the designated UAV selects and schedules the \emph{top-\(\alpha\)} devices for the next-round model aggregation. Let $S_t$ be the set of top-$\alpha$ devices. Its size $|S_t|=\lceil\alpha K\rceil$ is determined by the top $\alpha$ percent of devices with the largest $\|{\boldsymbol{g}}_k(t,\boldsymbol{\theta})\|_{2}$. In each round of global updates among UAVs, the following procedures are carried out sequentially.

\begin{itemize}
    \item \textbf{Intra-UAV aggregation}: Each UAV $u$ collects the base gradients $\{ {\boldsymbol{g}}_k(t,\boldsymbol{\theta}) \}_{k \in S_t \cap \mathcal{K}_u}$ of its serving edge devices, and computes an intermediate weighted sum:
    \begin{equation}
        {\mathbf{g}}_u = \sum_{k \in S_t \cap \mathcal{K}_u} |D_k| {\boldsymbol{g}}_k(t,\boldsymbol{\theta}),
        \label{equ:intra_UAV_agg}
    \end{equation}
    where $\mathcal{K}_u$ is the set of edge devices within the communication range of UAV $u$, and $|D_k|$ is the number of training samples of device $k$.
    
    \item \textbf{Inter-UAV aggregation}: The designated UAV collects ${\mathbf{g}}_u$ from all UAVs, and the global base model is updated via a weighted gradient step:
    \begin{equation}
        \boldsymbol{\theta}(t+1) = \boldsymbol{\theta}(t) - \eta_t \frac{1}{\sum_{k \in S_t} |D_k|} \sum_{u=1}^U {\mathbf{g}}_u.
        \label{equ:inter_UAV_agg}
    \end{equation}
    
    \item \textbf{Synchronization}: The updated global base model $\boldsymbol{\theta}(t+1)$ is broadcast to all UAVs, and then sent to devices for the next communication round.
\end{itemize}

Personalized parameters are updated locally at each device:
\begin{equation}
    {\boldsymbol{\phi}}_k(t+1) = {\boldsymbol{\phi}}_k(t) - \gamma_t {\mathbf{g}}_k(t,{\boldsymbol{\phi}}), \quad k = 1, \dots, K,
    \label{equ:personalized_local_update}
\end{equation}
with no aggregation.

\subsection{Computation and Communication Model}
In the studied UAV-aided personalized FL system, it is assumed that local edge devices have sufficient energy to complete the FL procedure. UAVs, however, have a limited energy budget. Therefore, in the following, we only consider the energy consumption of UAVs. The hovering time of UAVs in each communication round contributes to both computation time and communication time.

\subsubsection{Computation Model}

The computation within one communication round of the UAV-aided personalized FL system includes local training at edge devices and intra- and inter-aggregation at UAVs. The time consumption is illustrated as follows:

Each device trains a local model using its GPU. The local training time of device $k$ is
\begin{equation*}
T^{\mathrm{local}}_{k} = \frac{\tau \cdot |D_k|}{\nu}, 
\label{equ:local_training_time}
\end{equation*}
where $\nu$  denotes the per-device training throughput of GPU, $\tau$ represents the local epoch of model training at each device.

Each UAV performs model aggregation using its onboard CPU. Following the formulation in \cite{10550002}, the energy consumption for aggregation at UAV $u$ is given by
\begin{equation*}
E^{\mathrm{agg}}_{u} = \kappa\, q_u \,\varphi_u^{2} \Bigg( \sum\limits_{k} \rho_{k,u} + \mu_{u} \sum\limits_{u' \neq u} \rho_{u',u} \Bigg),
\label{equ:aggregation_energy}
\end{equation*}
where $\kappa$ is the energy consumption coefficient, $q_{u}$ is the number of computing cycles per data bit on UAV $u$, $\varphi_{u}$ is the CPU frequency of UAV $u$, $\rho_{k,u} = 1$ if device $k$ uploads its local model to UAV $u$ ($0$ otherwise), $\rho_{u',u} = 1$ if UAV $u'$ transmits its aggregated model to UAV $u$ ($0$ otherwise), and $\mu_{u}=1$ if UAV $u$ is selected to perform inter-UAV aggregation ($0$ otherwise). Notably, only the designated UAV performs the inter-aggregation operation.

The aggregation time at UAV $u$ is given by \cite{10550002} as
\begin{equation*}
T^{\mathrm{agg}}_{u} = \frac{q_u \Bigg( \sum\limits_{k} \rho_{k,u} + \mu_{u} \sum\limits_{u' \neq u} \rho_{u',u} \Bigg)}{\varphi_u},
\label{equ:aggregation_time}
\end{equation*}

\subsubsection{Communication Model}
The communication within a single round includes the exchange of local and global models. 
Local models are uploaded by edge devices via D2U communications. 
Specifically, we adopt realistic channel models for UAV-assisted federated learning based on \cite{zhang2019cellular,9453811}.
The D2U channel gain from device $k$ to UAV $u$ is modeled as
\begin{align*}
H_{k,u} = - \Big( P_{\mathrm{LoS},k} \cdot PL_{\mathrm{LoS},k} + P_{\mathrm{NLoS},k} \cdot PL_{\mathrm{NLoS},k} \Big),
\end{align*}
where
\begin{align*}
PL_{\mathrm{LoS},k} &= 20 \log \left(\frac{4 \pi f d_{k,u}}{c_0}\right) + \eta_{\mathrm{LoS}}, \\
PL_{\mathrm{NLoS},k} &= 20 \log \left(\frac{4 \pi f d_{k,u}}{c_0}\right) + \eta_{\mathrm{NLoS}}, \\
P_{\mathrm{LoS},k} &= \frac{1}{1 + a \exp(-b(r_k - a))}, \quad r_k = \sin^{-1} \frac{L_u}{d_{k,u}},
\end{align*}
$f$ is the carrier frequency, $c_0$ is the speed of light, and $d_{k,u}$ is the distance between device $k$ and UAV $u$.  

The delay for uploading the local model from device $k$ is
\begin{equation*}
T^{\mathrm{upload}}_{k} = \sum\limits_{u} \rho_{k,u} \frac{M_{\boldsymbol{\theta}}}{B_k \log_2 \left(1 + \frac{P_k^{\mathrm{max}} |H_{k,u}|}{B_k \sigma^2} \right)},
\end{equation*}
where $B_k = B / |S_t|$, assuming that a shared channel of total bandwidth $B$ is equally multiplexed among the scheduled edge devices. $P_k^{\mathrm{max}}$ is the transmit power of the edge device, and $\sigma^2$ is the noise power.

The local models collected by UAVs are then exchanged via U2U communications. 
The U2U channel is modeled as free-space path loss:
\begin{equation*}
H_{u,u'} = G (d_{u,u'})^{-2},
\end{equation*}
where $G$ is a constant power gain and $d_{u,u'}$ is the distance between UAVs $u$ and $u'$.  

These models capture the essential large-scale fading characteristics relevant to UAV-assisted networks, consistent with \cite{zhang2019cellular}. 
Small-scale fading effects are not explicitly modeled, as they are typically averaged out over the time scales relevant to federated learning aggregation.

The time to transmit intra-UAV aggregated models by UAV $u$ is
\begin{equation*}
T^{\mathrm{U2U}}_{u} = \frac{\max \limits_{k \in S_t \cap \mathcal{K}_u} \rho_{k,u} M_{\boldsymbol{\theta}}}{B_u \log_2 \left(1 + \frac{P_u^{\mathrm{max}} |H_{u,u'}|}{B_u \sigma^2} \right)},
\label{equ:U2U_transmit_time}
\end{equation*}
where $B_u = {B}/{\sum \limits_{u'}\sum \limits_{u'\neq u}\rho_{u',u}}$ with the assumption that a shared channel of bandwidth $B$ is equally multiplexed among all selected UAV pairs. The operator \(\max\) indicates the largest value among the devices, representing the bottleneck latency. Then, the corresponding energy consumption is
\begin{equation*}
E^{\mathrm{U2U}}_u = P_u^{\mathrm{max}} T^{\mathrm{U2U}}_u.
\label{equ:U2U_energy}
\end{equation*}
The time for synchronizing the global model by a designated UAV $u$ is
\begin{equation*}
\begin{aligned}
T^{\mathrm{broadcast}}_u = &\max \limits_{k} \frac{M_{\boldsymbol{\theta}}}{B \log_2 \left(1 + \frac{P_u^{\mathrm{max}} |H_{k,u}|}{B \sigma^2} \right)}\\ 
&+ \mu_u \cdot \max \limits_{u' \neq u} \frac{\rho_{u',u} M_{\boldsymbol{\theta}}}{B \log_2 \left(1 + \frac{P_u^{\mathrm{max}} |H_{u',u}|}{B \sigma^2} \right)},
\label{equ:global_broadcast_time}
\end{aligned}
\end{equation*}
the energy consumption is
\begin{equation*}
E^{\mathrm{broadcast}}_u = P_u^{\mathrm{max}} T^{\mathrm{broadcast}}_u.
\label{equ:global_broadcast_energy}
\end{equation*}

\subsubsection{Hovering Time and Energy}
In our system, UAVs remain at fixed locations throughout the entire communication rounds of the learning process. That is, UAVs do not move between training rounds, and any repositioning occurs outside the considered communication period. As a result, the energy consumption due to flight dynamics, propulsion, or trajectory changes is negligible during each round, and the dominant energy cost comes from hovering and communication-related operations.

The total hovering time of UAVs during one training round is
\begin{align*}
T_h = &\max \limits_k T^{\mathrm{local}}_k + \max \limits_k T^{\mathrm{upload}}_k + \\
&\max \limits_u T^{\mathrm{agg}}_u + \max \limits_u T^{\mathrm{U2U}}_u + \max \limits_u T^{\mathrm{broadcast}}_u,
\label{equ:hovering_time}
\end{align*}
so the corresponding energy consumption is
\begin{equation*}
E_h = P_h T_h,
\label{equ:hovering_energy}
\end{equation*}
where $P_h$ is the hovering power of the rotary-wing UAV.

Finally, the overall energy consumption of UAVs in the $t$-th communication round is
\begin{equation*}
E(t) = U \cdot E_h + \sum\limits_u \left( E^{\mathrm{agg}}_u + E^{\mathrm{U2U}}_u + E^{\mathrm{broadcast}}_u \right).
\label{equ:total_energy}
\end{equation*}

For convenience, we summarize key symbols in Table~\ref{tab:notation}.
\begin{table}[t]
\label{tab:notation}
\centering
\caption{Notation Summary}
\begin{tabular}{|c|l|}
\hline
\textbf{Symbol} & \textbf{Description} \\
\hline
$K$ & Total number of devices \\
$U$ & Total number of UAVs \\
$\mathcal{K}_u$ & Devices served by UAV $u$ \\
$S_t$ & Set of selected devices \\
$\alpha$ & ratio of selected devices\\
$\mathbf{g}_k(t,\cdot)$ & Gradient of device $k$ at round $t$ \\
$\mathbf{g}_u$ &Gradient aggregated on UAV $u$\\
$\boldsymbol{\theta}(t)$ & Global base model at round $t$\\
${\boldsymbol{\phi}}_k(t)$ & Personalized parameters of device $k$ at round $t$\\
$F_k(\boldsymbol{\theta},{\boldsymbol{\phi}}_k)$ & Local loss function of device $k$\\
$|D_k|$ &Local dataset size of device $k$\\
$M_{\boldsymbol{\theta}}$ & Size of global base model\\
$H_{u,u'}$ & U2U Channel gain\\
$H_{k,u}$ & D2U Channel gain\\
$\rho_{k,u}$ & Indicator of D2U communication\\
$\rho_{u',u}$ & Indicator of U2U communication\\
$E(t)$ & UAVs' total energy consumption at round $t$\\
$\eta$ & Learning rate \\
$\sigma^2$ & Noise power  \\
\hline
\end{tabular}
\label{tab:notation}
\end{table}

\section{Problem Formulation and Solution}
\subsection{Problem Formulation}
The design goal of this work is to maximize learning performance while adhering to UAV energy constraints. On one hand, collecting local models from spatially distributed devices through D2U and U2U communications incurs substantial energy consumption. On the other hand, data heterogeneity among devices significantly slows convergence, thereby requiring more communication rounds to achieve a desired accuracy. Thus, we formulate the following optimization problem to enhance training performance under strict UAV energy constraints:
\begin{subequations}\label{equ:optim_problem}
\begin{align}
\min_{\boldsymbol{\theta,\rho_{u',u},\rho_{u',u}}} \, & \frac{1}{K} \sum_{k=1}^{K} F_k\big(\boldsymbol{\theta}, {\boldsymbol{\phi}}_k^{\star}(\boldsymbol{\theta})\big) \\
&\text{s.t.} \, \sum_{t=1}^{T^{\text{max}}} E(t) \le \hat {E}.
\end{align}
\end{subequations}

where $T^{\text{max}}$ denotes the maximum number of communication rounds. Constraint (\ref{equ:optim_problem}b) ensures that the total energy consumption of UAVs does not exceed the energy budget $\hat {E}$.

\subsection{Solution Algorithm}

From a device scheduling perspective, the convergence analysis in Theorem~\ref{thm:nonconvex} reveals a trade-off between stochastic
variance, selection bias, and personalization-induced bias. Selecting more devices per round reduces stochastic variance by improving
the approximation of the global gradient, but it also increases communication overhead and may amplify personalization-induced bias due to heterogeneous device models. Under limited bandwidth and UAV energy budgets, scheduling all devices is therefore not optimal.

To balance these effects, we adopt a top-$\alpha$ device selection strategy, which prioritizes devices with the largest $\ell_2$-norm gradients.
Such devices contribute most to global model updates, allowing the system to reduce selection bias while maintaining a limited communication load.
Consequently, faster convergence can be achieved under the same resource constraints compared to random device selection.

The overall solution framework for problem~(\ref{equ:optim_problem}) is summarized in Algorithm~\ref{alg:uav-pfl}. The procedure combines
UAV-assisted communication with a bilevel personalized FL structure, where global backbone parameters are aggregated across devices while
personalized heads remain locally updated. UAVs serve as hierarchical aggregators, enabling efficient model coordination under limited energy
and bandwidth resources.

The algorithm begins with the initialization of a shared global base model $\boldsymbol{\theta}(0)$ and personalized parameters ${\boldsymbol{\phi}}_k(0)$ for each device $k$. At each communication round $t$, the following steps are executed:  

\begin{enumerate}
    \item \textbf{Local Gradient Computation:}  
    Each edge device computes two types of stochastic gradients:  
    ${\boldsymbol{g}}_k(t,\boldsymbol{\theta})$ w.r.t. the global base parameters $\boldsymbol{\theta}(t)$, and  
    ${\mathbf{g}}_k(t,{\boldsymbol{\phi}})$ w.r.t. the personalized parameters ${\boldsymbol{\phi}}_k(t)$.  
    This separation ensures that shared knowledge is collaboratively optimized, while personalization is retained locally.  

    \item \textbf{Top-$\alpha$ Device Scheduling:}  
    Only a subset $S_t$ of edge devices is scheduled in each round. Specifically, the top-$\alpha$ fraction of edge devices with the largest $\ell_2$-norms of ${\boldsymbol{g}}_k(t,\boldsymbol{\theta})$ are selected. 
    
    \item \textbf{UAV-level Aggregation:}  
    Each UAV $u$ collects gradients from its served edge devices in $S_t$ and performs a weighted aggregation proportional to local dataset sizes $|D_{k}|$. Since hierarchical aggregation performs an exact sample-weighted averaging, it introduces no additional bias while reducing the communication overhead.

    \item \textbf{Global Base Model Update:}  
    A designated UAV aggregates the UAV-level results and applies a global update to the base model following Eq. (\ref{equ:inter_UAV_agg}). The updated global parameters are transmitted back to all UAVs through U2U communications, and subsequently to edge devices by each serving UAV.

    \item \textbf{Local Personalization:}  
    In parallel, each edge device refines its personalized parameters ${\boldsymbol{\phi}}_k$ by descending along its own gradient direction, following Eq. (\ref{equ:personalized_local_update}). Since these updates are never shared, each edge device preserves adaptability to its local data distribution.  
\end{enumerate}
The iterations proceed until either the maximum communication rounds are reached or the UAV’s cumulative energy consumption meets the budget.

\begin{algorithm}[t]
\caption{UAV-aided Personalized Federated Learning}
\label{alg:uav-pfl}
\begin{algorithmic}[1]
\REQUIRE Number of edge devices $K$, UAVs $U$, local datasets $\{D_{k}\}$, learning rates $\eta_t, \gamma_t$, selection ratio $\alpha$.
\STATE Initialize global base model $\boldsymbol{\theta}(0)$ and personalized parameters ${\boldsymbol{\phi}}_k(0)$ for all $k$
\FOR{communication round $t = 0,1,2,\dots,T^{\text{max}}$}
    \FOR{each device $k = 1, \dots, K$ \textbf{in parallel}}
        \STATE Compute stochastic gradients following Eq. (\ref{eq:device:parameter:update})
    \ENDFOR
    
    \STATE Select subset $S_t \subseteq \{1,\dots,K\}$ of size $\lceil \alpha K \rceil$ with largest $\|{\boldsymbol{g}}_k(t,\boldsymbol{\theta})\|_{2}$
    
    \FOR{each UAV $u = 1, \dots, U$}
        \STATE Collect base gradients from associated devices:
        $\{{\boldsymbol{g}}_k(t,\boldsymbol{\theta})\}_{k \in S_t \cap \mathcal{K}_u}$
        \STATE Compute the intra-UAV aggregated gradient according to
        Eq.~(\ref{equ:intra_UAV_agg})
    \ENDFOR
    \STATE Randomly select one UAV with $\mathbf{g}_u \neq \mathbf{0}$ as the designated UAV for inter-UAV aggregation.
    \STATE The designated UAV collects $\{\mathbf{g}_u\}_{u=1}^{U}$ and computes the global base update according to
    Eq.~(\ref{equ:inter_UAV_agg})

    \STATE Broadcast updated base model $\boldsymbol{\theta}(t+1)$ to all UAVs and devices
    
    \FOR{each device $k = 1, \dots, K$ \textbf{in parallel}}
        \STATE Update personalized parameters locally following Eq. (\ref{equ:personalized_local_update})
    \ENDFOR
\ENDFOR
\end{algorithmic}
\end{algorithm}

Algorithm~\ref{alg:uav-pfl} adopts a gradient-based top-$\alpha$ device selection strategy primarily to improve learning efficiency. However, this strategy does not explicitly consider UAV communication delay or energy consumption. To address these practical constraints, we further introduce a communication-aware device and UAV selection scheme, summarized in Algorithm~\ref{alg:uav-device-sel}. This algorithm determines both the designated aggregation UAV and the participating devices by jointly accounting for device learning performance and communication overhead. Note that Algorithm~\ref{alg:uav-device-sel} acts as an enhancement to Algorithm~\ref{alg:uav-pfl} and does not modify the underlying learning procedure itself.

\begin{algorithm}[t]
\caption{Device and designated UAV selection}
\label{alg:uav-device-sel}
\begin{algorithmic}[1]
\REQUIRE 
Gradient of each device in round $t$: $\|{\boldsymbol{g}}_k(t,\boldsymbol{\theta})\|_{2}$

\STATE Rank all devices according to $\|{\boldsymbol{g}}_k(t,\boldsymbol{\theta})\|_{2}$ in descending order

\STATE Compute cumulative importance score of each UAV
    \[
    I_u = \sum_{k \in \mathcal{K}_u} \text{rank}(k)
    \]

\FOR{each UAV $u = 1,\dots,U$}
    \STATE Compute UAV priority score
    \[
    \psi_{u'} = \frac{\tilde{H}_{u,u'}}{\tilde{I}_u}
    \]
    \STATE Sort the UAVs in descending order of $\psi_{u'}$
    \STATE Starting from $\mathcal{U}_u = \{u\}$, greedily add UAVs in order of $\psi_{u'}$ until $\mathcal{U}_u$ covers at least $\lceil \alpha K \rceil$ devices  
    \STATE Compute transmission rates and total delay $Q_{u} = \max \limits_{u'\in\mathcal{U}_u} T^{\mathrm{U2U}}_{u'} + T^{\mathrm{broadcast}}_u + \sum \limits_{u'\in\mathcal{U}_u}\!\! I_{u'} +   I_{u}$
\ENDFOR

\STATE \textbf{Designated UAV selection:}
\[
u^\star = \arg\min_{u} Q_{u}
\]

\STATE \textbf{Device selection:}
\STATE Select UAV set $\mathcal{U}_{u^\star}$ according to $\psi_u$
\STATE Construct device set
\[
S_t = \bigcup_{u \in  \mathcal{U}_{u^\star}} \mathcal{K}_u
\]
\STATE Keep the first $\lceil \alpha K \rceil$ devices in $S_t$
 \ENSURE Selected device set $S_t$, selected UAV subset $\mathcal{U}_{u^\star}$, and designated UAV $u^\star$.  
\end{algorithmic}
\end{algorithm}

The objective of Algorithm~\ref{alg:uav-device-sel} is to jointly select an informative device subset and suitable UAVs while determining a designated UAV that minimizes the UAV communication delay (i.e., energy consumption) in each training round. First, devices are ranked in descending order according to their gradient norms $\|{\boldsymbol g}_k(t,\boldsymbol{\theta})\|_2$, which quantify their learning importance. For each UAV $u$, a cumulative importance score $I_u=\sum_{k\in\mathcal{K}_u}\mathrm{rank}(k)$ is computed based on the ranks of its associated devices. A UAV priority score $\psi_{u'}=\tilde{H}_{u,u'}/\tilde{I}_u$ is then calculated by combining normalized communication quality and importance, and UAVs with the highest priority are grouped to form a subset $\mathcal{U}_u$ that covers at least $\lceil \alpha K \rceil$ devices. For each candidate UAV acting as aggregation leader, the total utility $Q_u$ is evaluated. The designated UAV is then chosen as $u^\star=\arg\min_u Q_u$. Finally, the device set is constructed as $S_t=\bigcup_{u\in\mathcal{U}_{u^\star}}\mathcal{K}_u$, from which the top $\lceil \alpha K \rceil$ ranked devices are retained, yielding the selected device set $S_t$, UAV subset $\mathcal{U}_{u^\star}$, and designated UAV $u^\star$ for the current round.

\subsection{Convergence Analysis}
We now analyze the convergence behavior of the proposed algorithm.

\paragraph{Convergence Guarantee}
Under standard smoothness assumptions and sufficiently small stepsizes, the proposed algorithm converges to a stationary point of the global objective. Specifically, after $T$ communication rounds, the averaged squared gradient norm satisfies
\begin{equation}
\frac{1}{T}\sum_{t=0}^{T-1}
\mathbb{E}\!\left[\left\|\nabla F\big(\boldsymbol{\theta}(t)\big)\right\|^{2}\right]
=
\mathcal{O}\!\left(\frac{1}{T}\right)
+
\mathcal{O}\!\left(\eta\, V_{\boldsymbol{\theta}}\right)
+
\mathcal{O}\!\left(\big(\beta_{\mathrm{sel}}+\bar{E}\big)^{2}\right),
\end{equation}
where $V_{\boldsymbol{\theta}} := \sigma_{\boldsymbol{\theta}}^2/|S_t| + \sigma_{\rm samp}^2 + \beta_{\rm sel}^2$, which captures both stochastic gradient variance and selection-induced bias. The three terms in the convergence bound correspond to the optimization error, stochastic variance, and the coupled bias induced by device selection and personalization, respectively.

\paragraph{Key Factors in Convergence}
The convergence behavior is influenced by three categories of parameters:
\begin{enumerate}
    \item Optimization parameters $\eta$,
    \item Stochastic variance terms $(V_{\boldsymbol{\theta}}, |S_t|)$, and
    \item Bias terms induced by device selection and personalization $(\beta_{\rm sel}, \bar E)$.
\end{enumerate}
Among these, the only system-level parameter that can be controlled during training is the device selection set $S_t$, whose size $|S_t|$ and composition are determined by the scheduling policy.

The bound highlights a variance--bias trade-off in device selection. Increasing $|S_t|$ reduces the stochastic variance term ${\sigma_{\boldsymbol{\theta}}^2}/{|S_t|}$,
but may increase the bias term$(\beta_{\rm sel}+\bar E)^2$
if devices with poorly aligned personalized parameters are included. Consequently, selecting all devices is not always optimal under communication and energy constraints.

The selection bias $\beta_{\rm sel}$ captures the discrepancy between the aggregated gradient over the selected devices and the full-participation gradient. By prioritizing devices with the largest $\ell_2$-norm gradients, the proposed top-$\alpha$ strategy effectively minimizes this discrepancy for a given communication budget, thereby reducing $\beta_{\rm sel}$ in the convergence bound.

While the analytical techniques in Appendix follow standard nonconvex FL analysis, the resulting convergence bound reveals a previously uncharacterized trade-off between device selection and personalization bias in UAV-aided FL, which directly motivates the proposed scheduling strategy.

\subsection{Complexity Analysis}
The computational complexity of Algorithm~\ref{alg:uav-pfl} can be analyzed in terms of the number of edge devices $K$, UAVs $U$, and model parameters. Let $M_{\boldsymbol{\theta}}$ and $M_{\boldsymbol{\phi}}$ denote the number of parameters in the global base model $\boldsymbol{\theta}$ and in each personalized head ${\boldsymbol{\phi}}_k$, respectively. Let $n_k$ be mini-batch size per local training step at an edge device $k$. Algorithm~\ref{alg:uav-pfl} consists of the following computational units.
\begin{enumerate}
    \item \textbf{Local gradient computation:} Each edge device computes gradients using its local data, with complexity
    \[
        \mathcal{O}\big(\tau \cdot n_k (M_{\boldsymbol{\theta}} + M_{\boldsymbol{\phi}})\big).
    \]
    Since edge devices operate in parallel, the per-round complexity is dominated by the largest $n_k$.

    \item \textbf{Top-$\alpha$ device selection:} Selecting the top-$\alpha$ edge devices by sorting them by $\ell_2$-norm of gradients has complexity
    \[
        \mathcal{O}(K \log K).
    \]

    \item \textbf{Intra-UAV aggregation:} Each UAV $u$ aggregates gradients from its associated edge devices. For $K_u$ devices connected to UAV $u$, the complexity is
    \[
        \mathcal{O}(K_u M_{\boldsymbol{\theta}}),
    \]
    resulting in a total of $\mathcal{O}(K M_{\boldsymbol{\theta}})$ across all UAVs.

    \item \textbf{Inter-UAV aggregation :} The designated UAV aggregates gradients from $U$ UAVs, with complexity
    \[
        \mathcal{O}(U M_{\boldsymbol{\theta}}).
    \]

    \item \textbf{Personalized parameter update:} Each edge device updates ${\boldsymbol{\phi}}_k$ locally, with complexity
    \[
        \mathcal{O}(M_{\boldsymbol{\phi}}),
    \]
    performed in parallel across all edge devices.
\end{enumerate}
Therefore, the verall complexity per communication round is
\[
\mathcal{O}\Big(\max_k \tau \cdot n_k (M_{\boldsymbol{\theta}} + M_{\boldsymbol{\phi}}) + K \log K + (K+U) M_{\boldsymbol{\theta}} + M_{\boldsymbol{\phi}} \Big).
\]

Here, the hierarchical aggregation and parallel updates enable efficient scaling with the number of edge devices $K$.

\section{Simulation Results}
The simulation results are presented in this part to evaluate the proposed algorithm for the UAV-aided personalized FL system. In the simulation setup, $12$ UAVs are deployed in a $10 \times 10\,\text{km}^2
$ area. They are arranged along three lines—horizontal, vertical, and diagonal—with approximately equal numbers per line, maintaining a minimum separation of $600\,\text{m}$ to avoid overlap and ensure balanced area coverage. The height of UAVs is 100m. The ground coverage of each UAV is a circular area of 200m in radius. 
For completeness, we note that the energy consumption of rotary-wing UAVs can generally be modeled as a function of UAV speed~\cite{8663615}. In our scenario, UAVs remain stationary throughout the entire learning process, i.e., the UAV speed is zero. Consequently, the propulsion and flight energy terms reduce to the hovering power $P_h$ used in our model. Following the parameters in Zeng~\cite{8663615}, we set $P_h = 52.1$ W.
The constant power gains factor $G$ in U2U channel model is $-31.5$dB\cite{zhang2019cellular}.
The devices in each UAV's coverage area is uniformly distributed and the number of served edge devices per UAV is set as 2. The GPU ability of each edge device $\nu$ is 500 samples per epoch per second. The channel parameters for D2U communications are set as $a = 11.95$ and $b=0.14$, $\eta _{NLoS} = 23$dB, $\eta _{LoS} = 3$dB \cite{8038869}. 
 
For performance evaluation, we adopt the CIFAR-10 dataset~\cite{krizhevsky2009learning}, which consists of 60,000 color images of size $32 \!\times\! 32$ pixels across 10 classes, with 50,000 training samples and 10,000 test samples. Following the partitioning strategy in~\cite{arivazhagan2019federated}, the dataset is distributed among edge devices such that each device holds data from at most $c$ classes, where $c$ controls the degree of data heterogeneity. To eliminate the influence of dataset size on the gradients, all devices are assigned datasets of equal size. This setup ensures a fair comparison when evaluating the proposed scheme (hereafter referred to as \textbf{Proposed}).
The learning model is a convolutional neural network (CNN) consisting of two convolutional layers and three fully connected layers. The final fully connected layer is designated as the personalized parameter set. The size of the shared parameters $M_{\boldsymbol{\theta}}$ is $169{,}000 \times 8$ bits, while the total model size is $172{,}400 \times 8$ bits. Hence, the parameters of the personalized part account for only about $2\%$ of the total model size.
The batch size of local training is 10. The SGD optimizer is employed with a learning rate, i.e., $\eta = 0.005$ and $\gamma = 0.005$. All simulations are executed on an NVIDIA RTX 4090 GPU. 

Unless otherwise specified, the baseline comparisons employ the same top-$k$
selection strategy as Algorithm~1 to ensure fair evaluation of learning
performance. Algorithm~2 is separately evaluated to demonstrate improvements
in communication efficiency.

The major simulation parameters are listed in Table~\ref{table:sim-para}. \par 
\begin{table}[htbp]
\centering
\caption{Simulation Parameters}
\label{table:sim-para}
\begin{tabular}{ll}
\toprule
\textbf{Parameter} & \textbf{Value} \\
\midrule
Carrier frequency & 2 GHz \\
Bandwidth & 80 MHz \\
CPU cycles of UAV $q_u$ & 20,000 \cite{jiang2024over} \\
CPU frequency of UAV $\boldsymbol{\phi}_u$ & 3 GHz \\
Energy consumption coefficient $\kappa$ & $10^{-27}$ \\
Energy budget of UAV $\hat{E}$ & 40 kJ \\
Noise power $\sigma^2$ & $-174$ dBm/Hz \\
Hover power of UAV  \\
Communication power of UAV $P_{\max}^u$ & 5 W \\
Communication power of device $P_{\max}^k$ & 23 dBm \\
Local epoch of device $\tau$ & 1 \\
Maximum communication rounds $T^{\max}$ & 210 \\
\bottomrule
\end{tabular}
\end{table}

The following existing schemes are compared with ours to show the efficiency of our proposed scheme in the general UAV-aided FL architecture. 
\begin{itemize}
    \item \textbf{Grad-Based FedAvg}~\cite{jiang2024over}: schedules local devices corresponding to the top half of all devices in terms of gradient $\ell_2$-norms and updates the global model via FedAvg.
    \item \textbf{Inter-UAV FedAvg}~\cite{10550002}: each UAV aggregates models from devices within its communication range, then a designated UAV collects these aggregated models from all other UAVs to produce a global model via FedAvg. The global model is subsequently broadcast to all UAVs via U2U communication, and each UAV then distributes it to its served devices.
    \item \textbf{Intra-UAV FedAvg}~\cite{9453811}: each UAV aggregates local models from its served devices to produce a global model via FedAvg, which is then directly broadcast to its served devices.
    \item \textbf{Reliable Personalized FL (RIPFL)}~\cite{10204293}: aggregates personalized global models for each local device based on cosine similarity. The number of aggregated models depends on the device uncertainty: higher uncertainty leads to aggregation of more models.
    \item \textbf{Sequential FedPer}~\cite{arivazhagan2019federated}: differs from the proposed scheme only in that edge devices are scheduled in a round-robin sequential manner with the same number of models as in the proposed scheme, rather than based on gradient $\ell_2$-norms.
    \item \textbf{Random FedPer}~\cite{arivazhagan2019federated}: differs from the proposed scheme only in that edge devices are scheduled randomly with the same number of models as in the proposed scheme, rather than based on gradient $\ell_2$-norms.
\end{itemize}

Although the theoretical convergence rates characterize worst-case behavior, we empirically examine whether the predicted trends hold in practice. In particular, we study the impact of device selection size on convergence speed. As predicted by Theorem~\ref{thm:nonconvex}, selecting a moderate number of devices achieves faster convergence than both full participation and local-only training, confirming the bias–variance trade-off revealed by the analysis. The result in Fig.~\ref{fig:robustness-mobility-high}, Fig.~\ref{fig:robustness-mobility-low}, and Fig.~\ref{fig:impact of FL} verify that convergence behavior under different device participation levels.
\begin{figure}[htpb]
	\centering
\subfloat[Fixed device locations]{
    \includegraphics[width=2.5in]{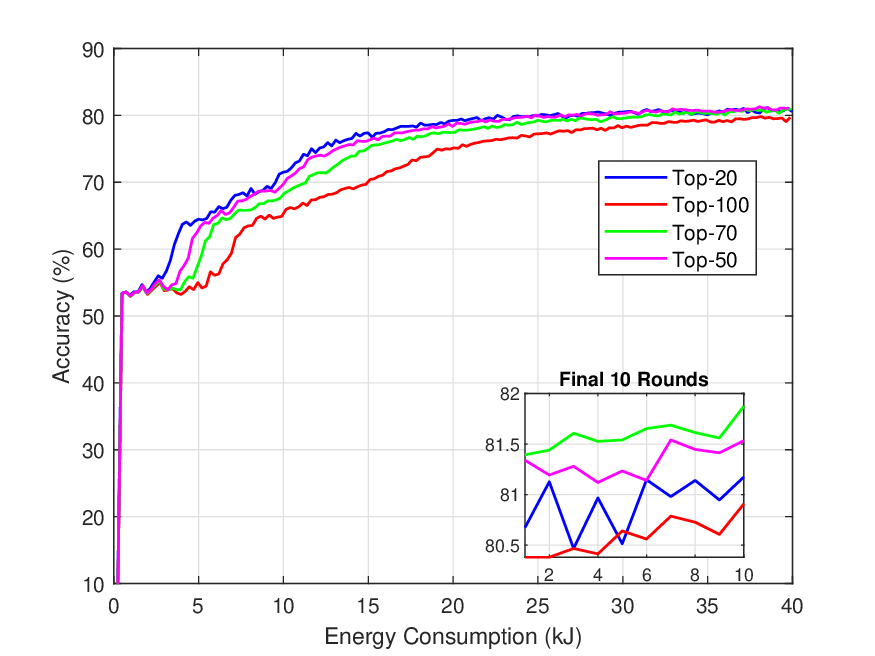}}\\
\subfloat[Mobile device locations]{
    \includegraphics[width=2.5in]{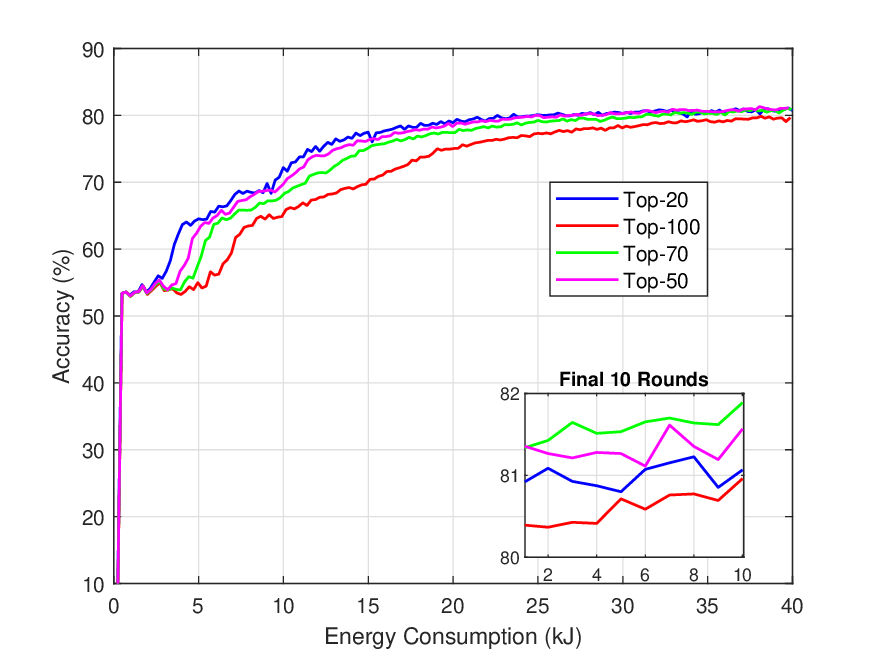}}

\caption{Robustness of the proposed scheme under device mobility in high-ratio device scheduling.}
\label{fig:robustness-mobility-high}
\end{figure} 

Fig.~\ref{fig:robustness-mobility-high} illustrates the robustness of the proposed top-$\alpha$ scheduling scheme under device mobility in the data-heterogeneous scenario with $c=6$. Here, top-$\alpha$ denotes that only the top $\alpha \%$ fraction of devices, ranked by the gradient $\ell_2$-norm, are scheduled for global aggregation in each round. Specifically, Fig.~\ref{fig:robustness-mobility-high}(a) corresponds to the case of fixed device locations, where each device remains at its initially generated position throughout the training process, while Fig.~\ref{fig:robustness-mobility-high}(b) corresponds to mobile device locations, where device positions are randomly generated within the UAV coverage area in every communication round, independently across rounds.  

From both figures, we observe that the convergence speed increases as the scheduling ratio $\alpha$ decreases (i.e., $\text{Top-100} < \text{Top-70} < \text{Top-50} < \text{Top-20}$).  From the inset of Fig.~\ref{fig:robustness-mobility-high}, full-participation scheduling consistently achieves lower accuracy by approximately $1\%$ than partial scheduling with smaller $\alpha$, under both fixed and mobile device location scenarios. This observation is consistent with our theoretical convergence analysis. Specifically, by selecting only the top-$\alpha$ devices according to the $\ell_2$-norm of their gradients, the global backbone update is guided by the most informative optimization directions, while the influence of noisy or low-quality updates is suppressed. Consequently, the variance of the aggregated gradient is reduced, resulting in faster convergence.

Furthermore, since the comparison is made in terms of energy consumption versus test accuracy, it is evident that the proposed top-$\alpha$ scheme consistently achieves higher accuracy with lower energy expenditure compared to scheduling all edge devices. This is because scheduling fewer but more informative devices reduces both communication overhead and U2U transmission cost. 

\begin{figure}[htpb]
	\centering
\subfloat[Fixed device locations]{
    \includegraphics[width=2.5in]{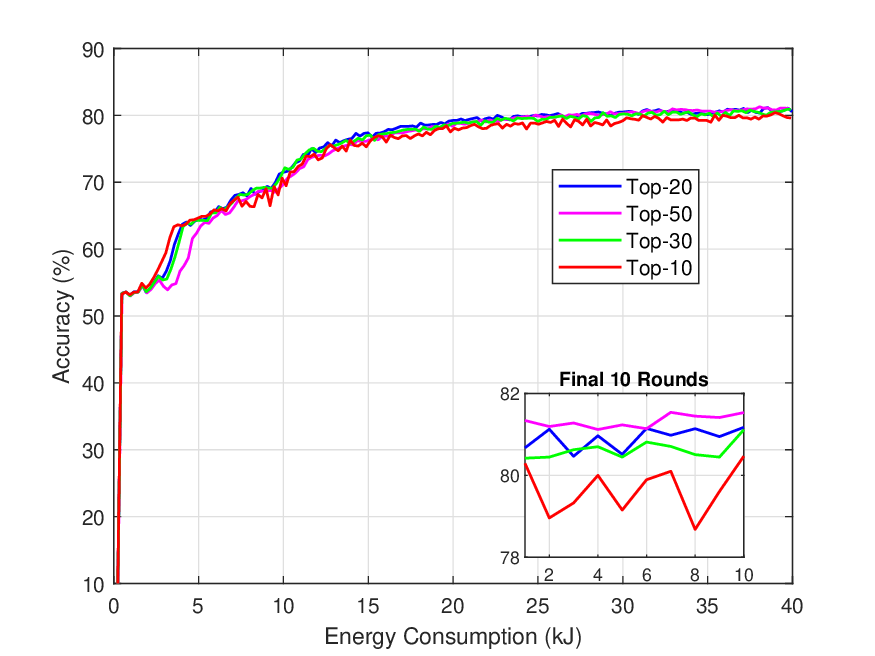}}\\
\subfloat[Mobile device locations]{
    \includegraphics[width=2.5in]{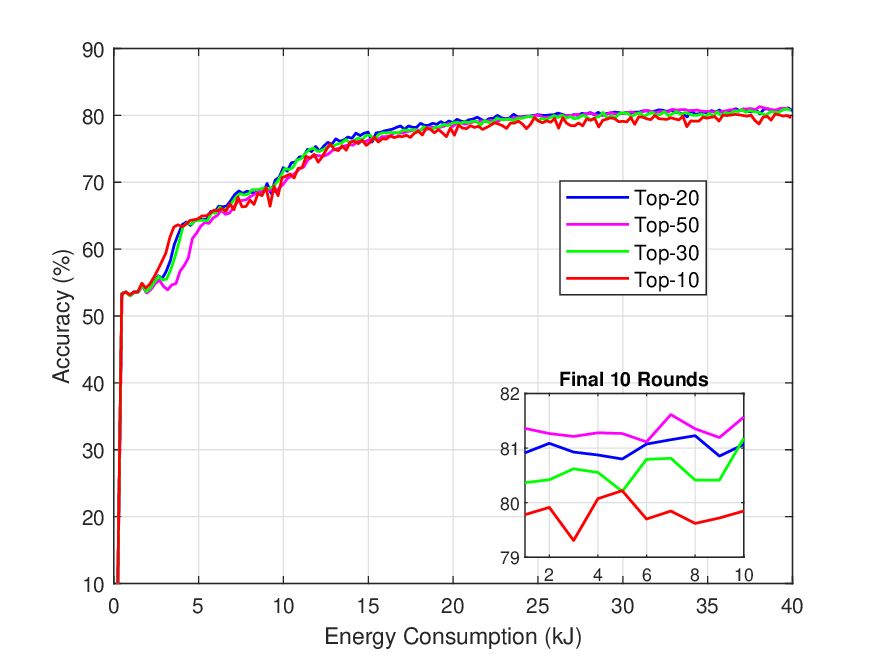}}

\caption{Robustness of the proposed scheme under device mobility in low-ratio device scheduling.}
\label{fig:robustness-mobility-low}
\end{figure} 

Fig.~\ref{fig:robustness-mobility-low} illustrates the robustness of the proposed top-$\alpha$ scheduling scheme under device mobility in the low-ratio scenario, with data heterogeneity set to $c=6$, consistent with Fig.~\ref{fig:robustness-mobility-high}. It can be observed that Top-50, Top-30, and Top-20 exhibit comparable convergence speeds, whereas further reducing the scheduling ratio to Top-10 leads to degraded performance compared with Top-20.  From the inset of Fig.~\ref{fig:robustness-mobility-low}, the Top-10 scheme achieves approximately $1.5\%$ lower accuracy than Top-50 and about $1\%$ lower accuracy than Top-20. This performance degradation arises because selecting only the top $10\%$ of devices may exclude moderately informative updates, thereby weakening the variance reduction effect in gradient aggregation. Among all evaluated scheduling ratios, Top-20 achieves the best tradeoff by enabling fast convergence while avoiding aggregation from less informative devices, thus improving both learning performance and energy efficiency.

Combining the observations from Fig.~\ref{fig:robustness-mobility-high} and Fig.~\ref{fig:robustness-mobility-low}, the results validate the advantage of the proposed gradient-norm--based scheduling strategy: concentrating aggregation on the most informative devices accelerates convergence, reduces communication overhead, and ensures robust performance under both fixed and mobile device locations. In all subsequent evaluations, the device mobility scenario is assumed, and the scheduling ratio $\alpha$ is set to 20\%.

\begin{figure}[htpb]
	\centering
	\includegraphics[width=2.5in]{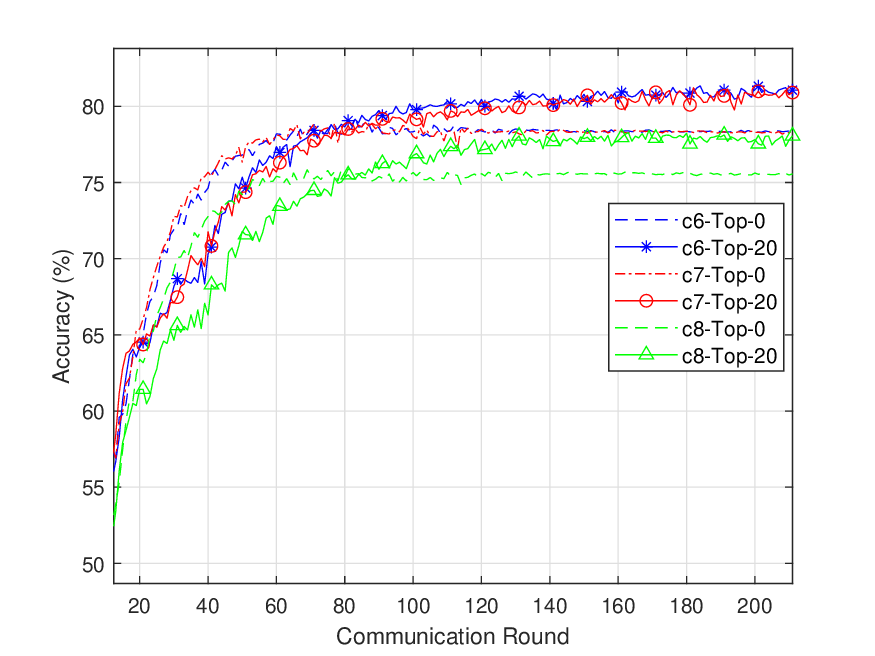}\\
\caption{Impact of FL on the test accuracy under different levels of data heterogeneity.}
\label{fig:impact of FL}
\end{figure} 

Fig.~\ref{fig:impact of FL} compares the test accuracy of devices with and without FL under different levels of data heterogeneity ($c=6$, $c=7$, and $c=8$). Here, Top-0 corresponds to the case where no devices are scheduled for global aggregation, and all devices rely solely on local training without collaboration. In contrast, Top-20 enables federated aggregation among selected devices. It is observed that Top-20 consistently achieves significantly higher test accuracy than Top-0 across all heterogeneity levels. These results confirm that, under the considered heterogeneous settings, FL substantially enhances learning performance compared to purely local training.
The consistent performance gain across different values of $c$ further indicates that the benefit of FL is robust to variations in data heterogeneity.

\begin{figure*}[htbp]
    \centering
    \subfloat[$c=6$]{
        \includegraphics[width=0.32\textwidth]{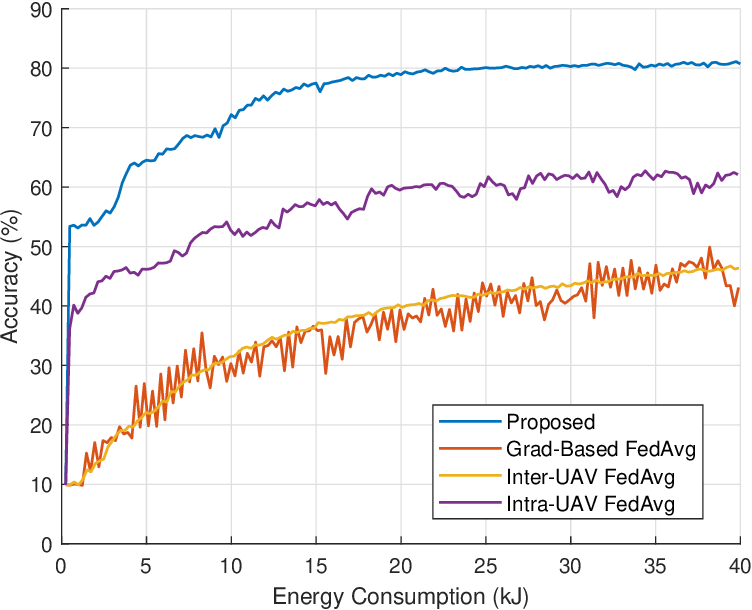}}
    \hfill
    \subfloat[$c=7$]{
        \includegraphics[width=0.32\textwidth]{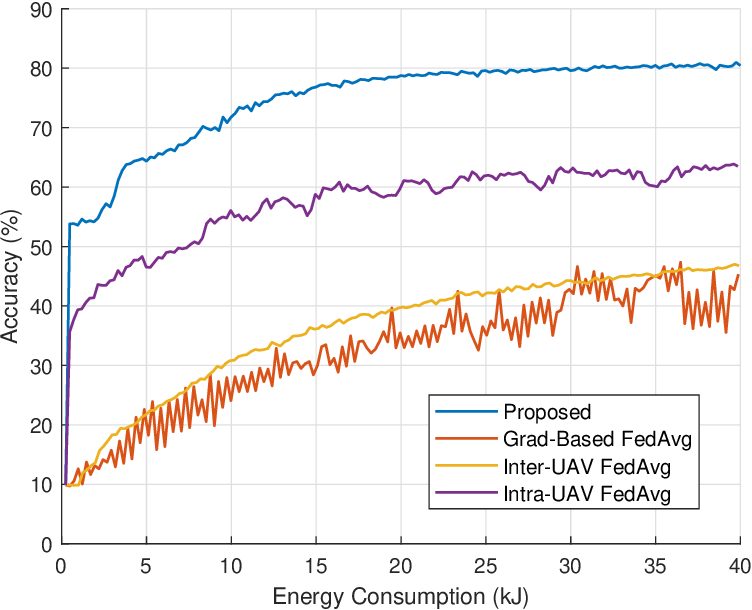}}
    \hfill
    \subfloat[$c=8$]{
        \includegraphics[width=0.32\textwidth]{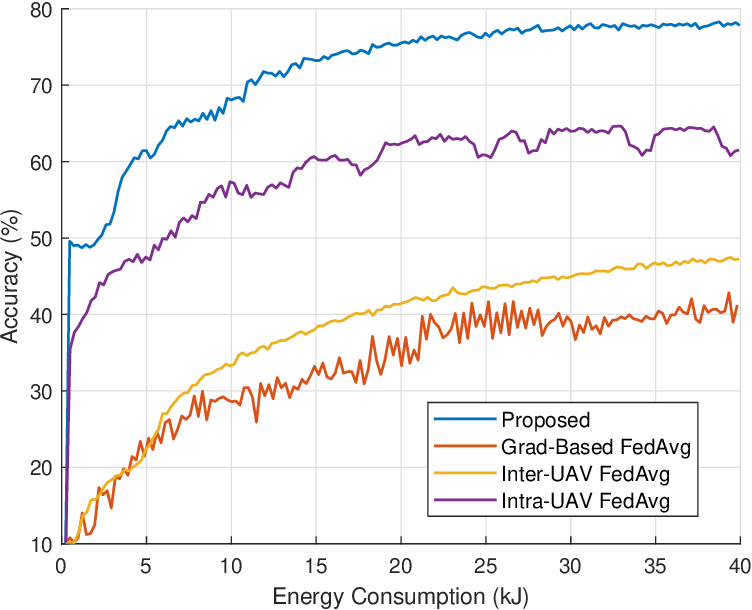}}
    \caption{Comparison of performance under different UAV-aided FL schemes.}
    \label{fig:impact of FedPer}
\end{figure*}

Fig.~\ref{fig:impact of FedPer} compares the proposed scheme with \textbf{Grad-Based FedAvg}~\cite{jiang2024over}, \textbf{Inter-UAV FedAvg}~\cite{10550002}, and \textbf{Intra-UAV FedAvg}~\cite{9453811} under different levels of data heterogeneity ($c=6$, $c=7$, and $c=8$). The proposed method consistently outperforms these state-of-the-art UAV-aided FL approaches. Unlike conventional schemes that update the entire model via FedAvg, the performance gain of the proposed method stems from its \emph{shared-backbone with local personalization} design.  

In heterogeneous data scenarios, the global backbone—trained primarily from strong gradient signals—serves as a robust feature extractor across devices, while the lightweight device-specific heads adapt to local distributional variations. This separation mitigates negative transfer, since the backbone is not forced to overfit each device’s private data distribution, thereby improving per-device accuracy. Moreover, the backbone is updated using the most ``active'' devices in each round, effectively steering representation learning toward the most informative features in the population. The personalized layers then absorb device-specific characteristics (e.g., label bias or stylistic variations), substantially alleviating the detrimental impact of data heterogeneity.

\begin{figure*}[htbp]
    \centering
    \subfloat[$c=6$]{
        \includegraphics[width=0.32\textwidth]{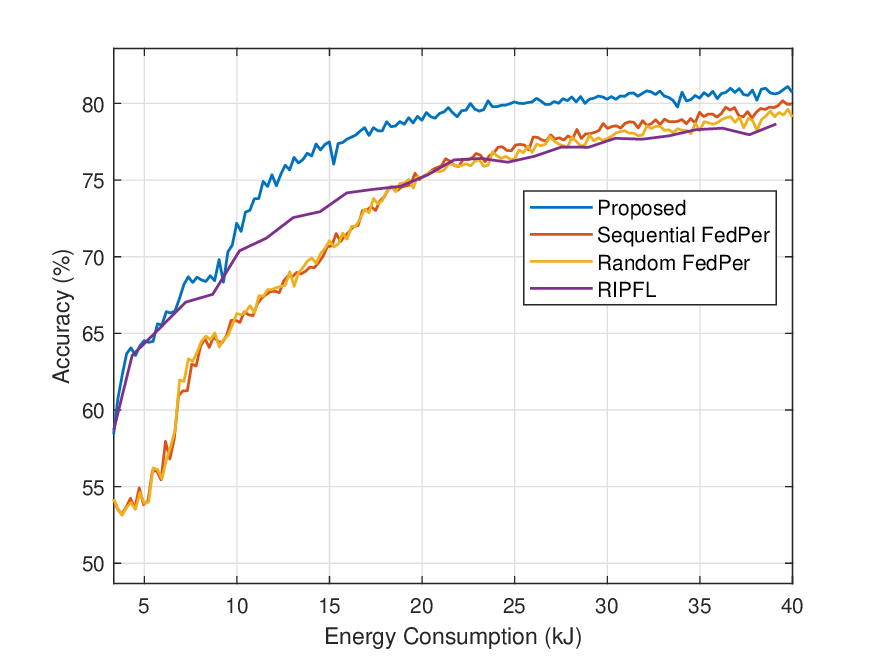}}
    \hfill
    \subfloat[$c=7$]{
        \includegraphics[width=0.32\textwidth]{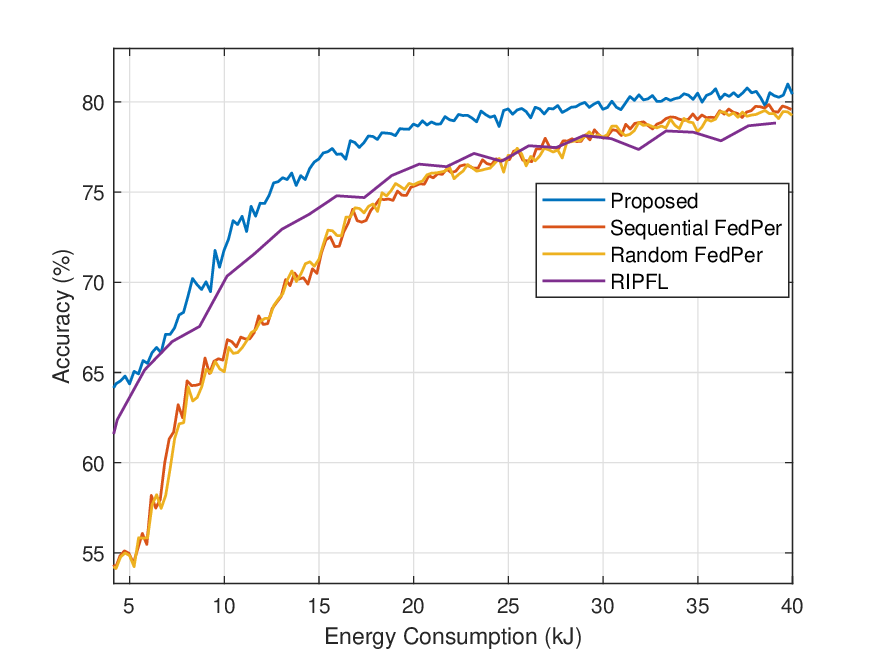}}
    \hfill
    \subfloat[$c=8$]{
        \includegraphics[width=0.32\textwidth]{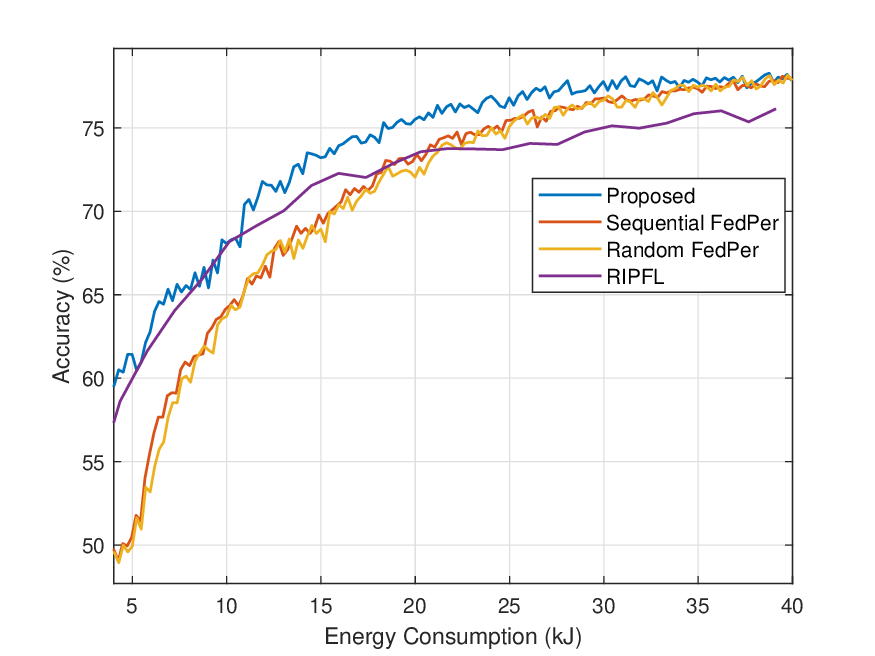}}
    \caption{Comparison of different device scheduling schemes for the case of one local epoch ($\tau=1$).}
    \label{fig:impact of grad-based scheduling-e1}
\end{figure*}

Fig.~\ref{fig:impact of grad-based scheduling-e1} compares the impact of different device scheduling schemes on convergence speed under varying levels of data heterogeneity, i.e., $c=6$, $c=7$, and $c=8$. It is observed that the proposed scheduling scheme consistently achieves faster convergence than the baselines. The underlying reason is that by aggregating only the top 20$\%$ of devices with the largest gradient norms for the backbone, the shared update is biased toward devices providing the strongest learning signals. This avoids diluting the global direction with weak or conflicting gradients, thereby accelerating the representation learning of shared features.  

In contrast, the similarity-based scheduling strategy adopted in \textbf{RIPFL}~\cite{10204293} exhibits slower convergence. \textbf{RIPFL} constructs an individual global model for each device by aggregating others' updates and then computing pairwise model similarity. This mechanism introduces substantially higher communication overhead compared with the centralized aggregation in the proposed scheme, particularly in multi-UAV aided FL systems deployed over large geographical areas. Moreover, as the task complexity increases (i.e., larger $c$), the early training rounds become more uncertain, making model similarity more prone to providing misleading scheduling signals.  

By comparison, the proposed gradient-norm based selection directly exploits an optimization-relevant indicator—the gradient magnitude—which is causally related to learning progress. This offers a more reliable criterion for identifying devices with informative updates, thereby reducing redundant communication, improving scheduling efficiency, and accelerating global model convergence.

\begin{figure*}[htbp]
    \centering
    \subfloat[$c=6$]{
        \includegraphics[width=0.32\textwidth]{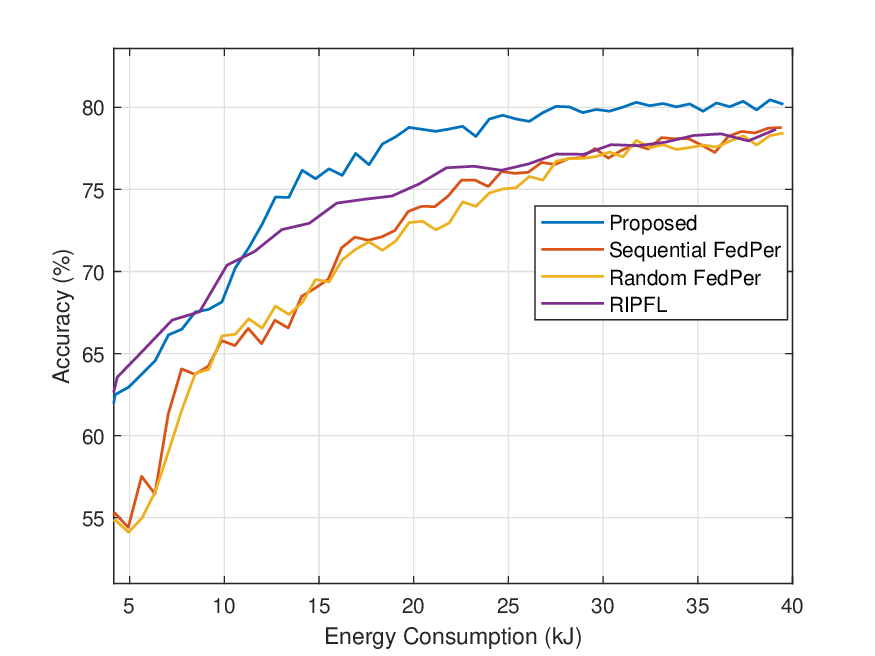}}
    \hfill
    \subfloat[$c=7$]{
        \includegraphics[width=0.32\textwidth]{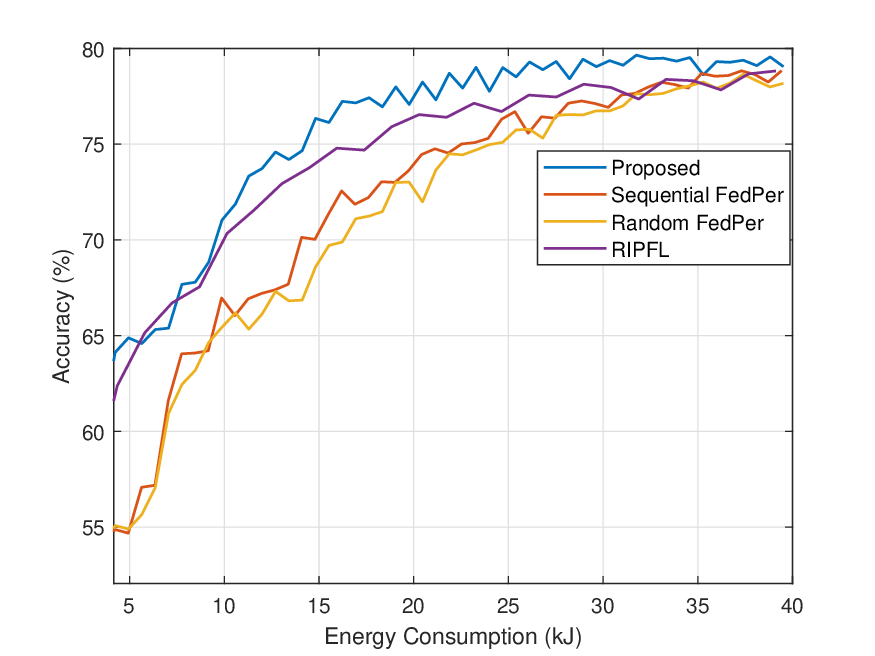}}
    \hfill
    \subfloat[$c=8$]{
        \includegraphics[width=0.32\textwidth]{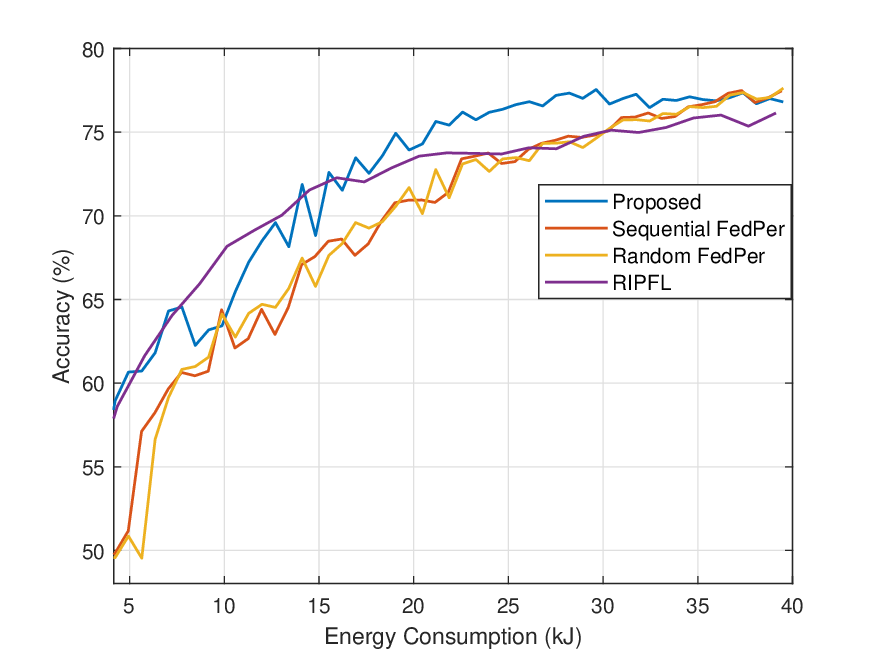}}
    \caption{Comparison of different device scheduling schemes for the case of three local epochs ($\tau=3$).}
    \label{fig:impact of grad-based scheduling-e3}
\end{figure*}

Fig.~\ref{fig:impact of grad-based scheduling-e3} presents the same comparison of scheduling schemes as in Fig.~\ref{fig:impact of grad-based scheduling-e1}, but under the setting where the local epoch is $3$. For reference, in Fig.~\ref{fig:impact of grad-based scheduling-e1}, the local epoch was set to $1$. 
Since \textbf{RIPFL} exhibits severely degraded performance when the local epoch is set to $1$, we retain the setting of $6$ epochs for \textbf{RIPFL}. The cases for \textbf{RIPFL} with $3$ and $10$ local epochs show slightly worse convergence compared to the $6$-epoch setting. Therefore, for clarity and fairness of comparison, we report the results of \textbf{RIPFL} with $6$ local epochs, while the other three schemes are evaluated with $3$ local epochs.

From Fig.~\ref{fig:impact of grad-based scheduling-e3}, we observe that the learning performance of the proposed scheme is only weakly affected by the number of local epochs. Although increasing the local epoch may exacerbate client drift in heterogeneous settings, the resulting performance degradation is negligible for the proposed method. This suggests that the gradient-based scheduling mechanism effectively controls the adverse effects of excessive local updates. Furthermore, for all examined levels of data heterogeneity, the proposed approach consistently converges faster than the compared schemes, highlighting its robustness and stability across a wide range of system configurations.

\begin{figure}[htpb]
	\centering
	\subfloat[]{
		\includegraphics[width=2.5in]{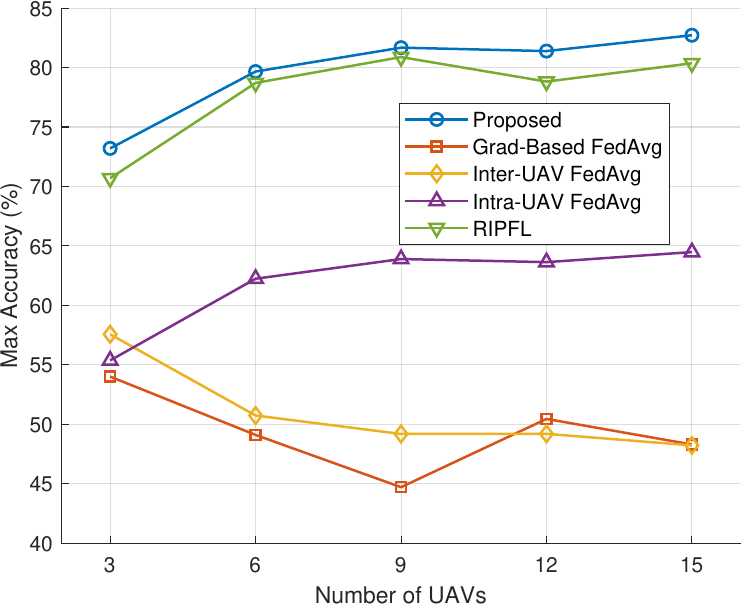}}\\
	\subfloat[]{
		\includegraphics[width=2.5in]{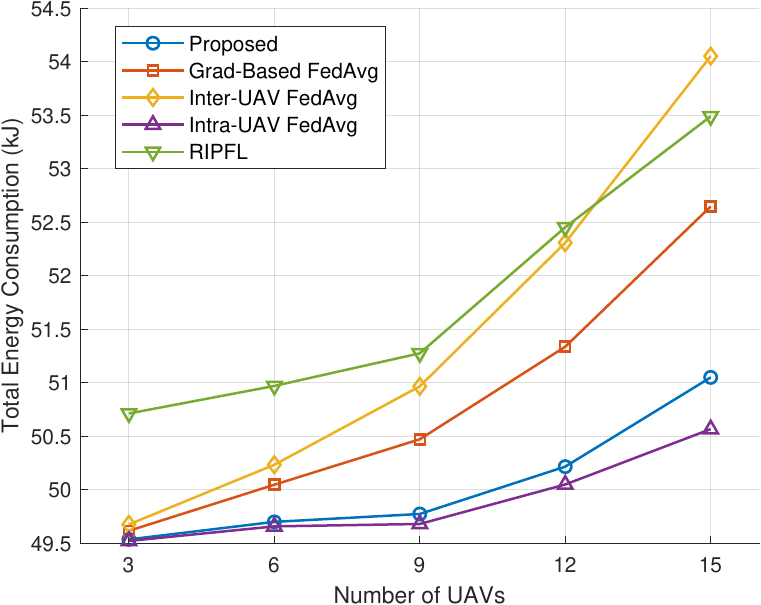}}
	\caption{Comparison of different schemes with varying number of UAV.}
    \label{fig:impact of numUAV}
\end{figure} 
Fig.~\ref{fig:impact of numUAV} illustrates the performance of the proposed scheme under data heterogeneity $c=6$ with varying numbers of UAVs. Fig.~\ref{fig:impact of numUAV}(a) presents the maximum accuracy achieved after $210$ local training epochs, while Fig.~\ref{fig:impact of numUAV}(b) shows the corresponding total energy consumption. It can be observed that regardless of the number of UAVs, the accuracy of the proposed scheme consistently outperforms the baseline methods while requiring only a modest amount of energy. Although the \textbf{Intra-UAV FedAvg}~\cite{9453811} achieves slightly lower energy consumption than the proposed scheme, its accuracy is significantly inferior. On the other hand, \textbf{RIPFL} yields slightly lower accuracy than the proposed scheme but at the cost of substantially higher energy consumption. Furthermore, as the number of UAVs increases, the performance gap between the proposed scheme and \textbf{RIPFL} further widens. This is because the gradient-based top-$\alpha$ scheduling becomes more effective with more UAVs: a larger set of candidate devices per round allows for the selection of more informative devices based on gradient $\ell_{2}$-norm, thereby reducing gradient variance and improving global model convergence. Consequently, the proposed scheme simultaneously enhances learning accuracy and energy efficiency as the UAV fleet size grows.

The above results are obtained using only the top-$\alpha$ device scheduling scheme, with UAVs selected randomly (denoted as “Max-Grad + Random UAV” in Fig.~\ref{fig:opt-UAVcomm-sim}). To evaluate the performance of the proposed joint device and UAV selection optimization in Algorithm~\ref{alg:uav-device-sel} (denoted as “Grad-Energy Tradeoff” in Fig.~\ref{fig:opt-UAVcomm-sim}), we consider the following baselines. “Max-Grad + Energy-Opt” refers to the approach where devices are first selected according to the top-$\alpha$ gradient criterion, and then the designated UAV for aggregation is chosen to minimize energy consumption based on UAV communication delay. “Energy-Opt Only” represents the strategy where UAVs are selected solely to minimize energy consumption, and device selection is subsequently performed based on the available UAVs.

From Fig.~\ref{fig:opt-UAVcomm-sim}, it can be observed that the proposed selection optimization scheme achieves over $3\%$ higher learning accuracy compared to the ``Energy-Opt Only'' scheme, with an increase of less than 0.1~kJ in energy consumption. Furthermore, the proposed scheme reaches approximately the same maximum accuracy while consuming less energy. These results validate the effectiveness of the proposed device and UAV selection scheme in balancing learning performance and energy consumption.

To further evaluate the proposed algorithm in a realistic UAV-aided disaster-resilient scenario, we consider a UAV-covered area where devices are distributed according to a homogeneous Poisson point process (PPP) with an intensity of $0.5\times10^{-4}$ devices/m\textsuperscript{2} in the 2-dimensional Euclidean space \cite{hydher2020intelligent}. As shown in Fig.~\ref{fig:opt-UAVcomm-real}, the proposed device and UAV selection scheme again demonstrates its efficiency in trading off learning performance and energy consumption.

\begin{figure}[htpb]
	\centering
\subfloat[]{
    \includegraphics[width=2.5in]{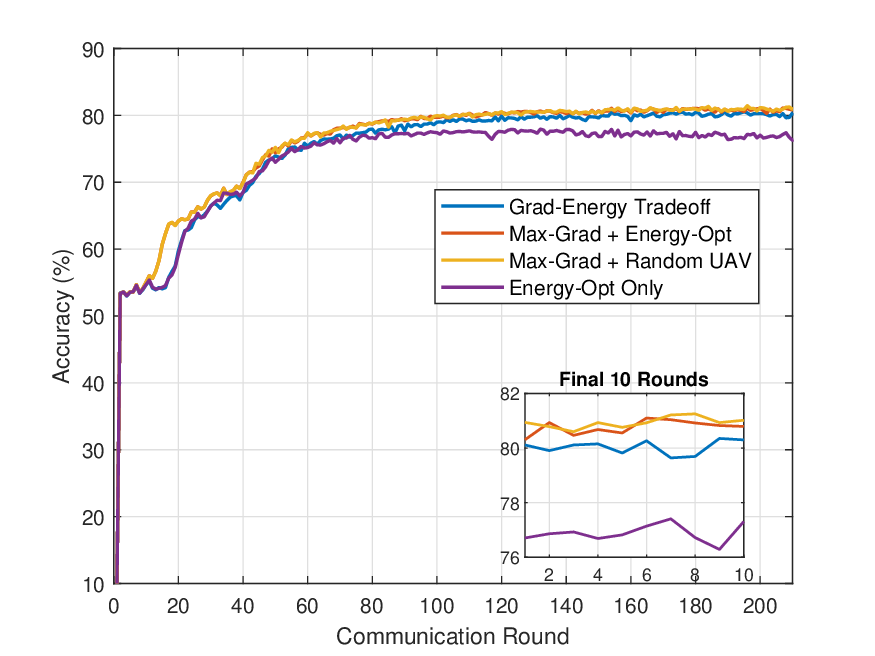}}\\
\subfloat[]{
    \includegraphics[width=2.5in]{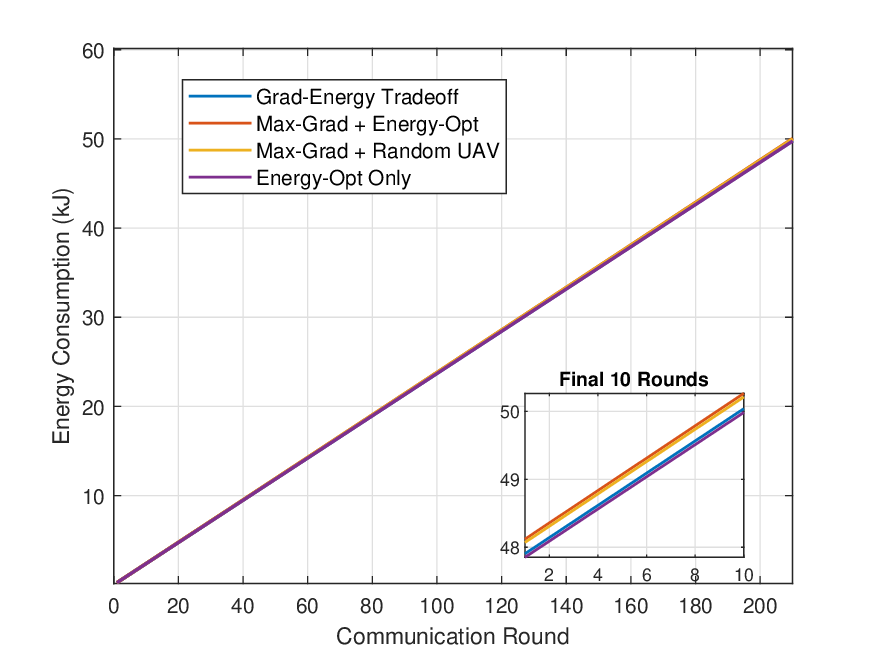}}

\caption{Performance comparison of selection optimization under uniform device distribution.}
\label{fig:opt-UAVcomm-sim}
\end{figure}

\begin{figure}[htpb]
	\centering
\subfloat[]{
    \includegraphics[width=2.5in]{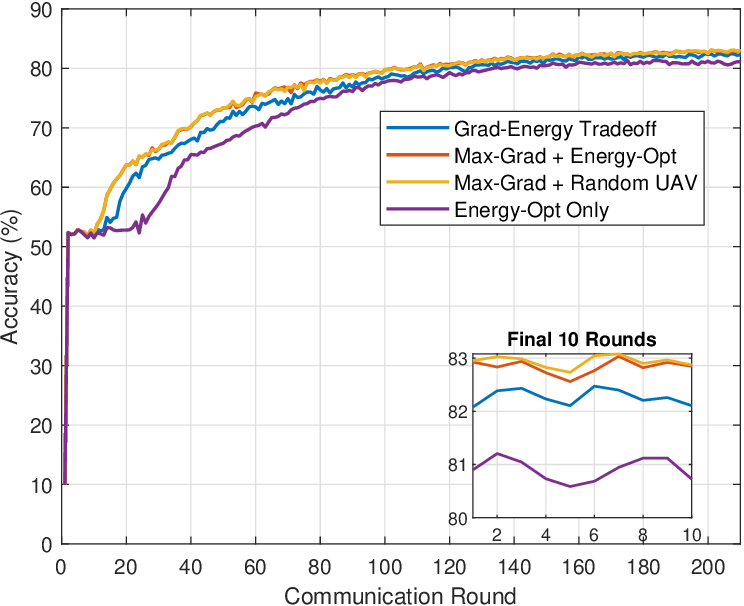}}\\
\subfloat[]{
    \includegraphics[width=2.5in]{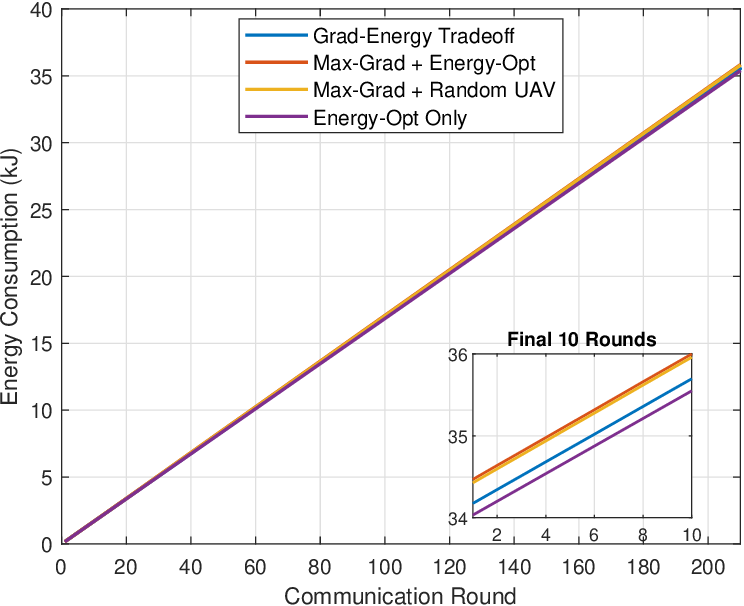}}

\caption{Performance comparison of selection optimization under PPP device distribution.}
\label{fig:opt-UAVcomm-real}
\end{figure}

\section{Conclusion}
In this work, we investigated a UAV-aided personalized FL system under data heterogeneity. By adopting a strict separation between the globally shared backbone and the permanently local personalization heads, the proposed framework effectively addressed the challenges associated with heterogeneous data distributions. To further balance the trade-off between selection bias and personalization bias, while reducing communication overhead and energy consumption, a top-$\alpha$ device scheduling strategy based on the largest gradient $\ell_{2}$-norm was developed. Simulation results demonstrated that the proposed scheme consistently outperforms existing UAV-aided FL approaches across different levels of data heterogeneity, achieving both higher accuracy and lower energy cost. 

For future work, this study will be extended to the design of energy-efficient resource allocation schemes for UAV-enabled model exchange in personalized FL systems. In addition, trajectory or position optimization of UAVs is an interesting and complementary direction that can further enhance communication quality and learning performance, which we leave for future investigation.

%
\appendix[Convergence Analysis of UAV-aided personalized FL]

Since intra-/inter-UAV steps realize an exact sample-weighted average, they introduce no additional bias. Therefore, the UAV-aided aggregation performs the sample-weighted average can also express
\begin{align*}
    &\overline{\mathbf{g}}_{S_t}{(\,t,\boldsymbol{\theta})}
    :=\frac{\sum_{k\in S_t}|D_k|\,{\boldsymbol{g}}_k(t,\boldsymbol{\theta})}{\sum_{k\in S_t}|D_k|}, \\
    &\boldsymbol{\theta}(t+1)=\boldsymbol{\theta}(t)-\eta_t\,\overline{\mathbf{g}}_{S_t}{(\,t,\boldsymbol{\theta})},
\end{align*}
and the personalized parameters are updated locally (never aggregated):
\begin{equation}
    {\boldsymbol{\phi}}_k(t+1)={\boldsymbol{\phi}}_k(t)-\gamma_t\,{\mathbf{g}}_k(t,{\boldsymbol{\phi}}).
    \label{eq:phi-update}
\end{equation} 
\subsubsection*{Decomposition and Bias Definitions}
Let
\begin{align*}
&{\mathbf{g}}_k^\star(\boldsymbol{\theta}):=\nabla_{\boldsymbol{\theta}} F_k\!\big(\boldsymbol{\theta},{\boldsymbol{\phi}}_k^\star(\boldsymbol{\theta})\big), \\
    &\Delta_k(t):={\boldsymbol{\phi}}_k(t)-{\boldsymbol{\phi}}_k^\star(\boldsymbol{\theta}(t)).
\end{align*}
By adding and subtracting $\nabla_{\boldsymbol{\theta}} F_k(\boldsymbol{\theta}(t),{\boldsymbol{\phi}}_k^\star(\boldsymbol{\theta}(t)))$,
\begin{equation}
\begin{aligned}
    &{\mathbf{g}}_k(t,\boldsymbol{\theta})={\mathbf{g}}_k^\star(\boldsymbol{\theta}(t))+\mathbf{e}_k(t)+{\mathbf{\zeta}}_k(t),\\
    &\mathbf{e}_k(t):=\nabla_{\boldsymbol{\theta}} F_k(\boldsymbol{\theta}(t),{\boldsymbol{\phi}}_k(t))-\nabla_{\boldsymbol{\theta}} F_k(\boldsymbol{\theta}(t),{\boldsymbol{\phi}}_k^\star(\boldsymbol{\theta}(t))).
     \end{aligned}
    \label{eq:grad-decomp}
\end{equation}
We use the averaged quantities (weights proportional to $|D_k|$):
\begin{equation}
\begin{aligned}
    \bar {\mathbf{g}}_{S_t,\star} &:= \frac{\sum_{k\in S_t} |D_k|\, {\mathbf{g}}_k^\star(\boldsymbol{\theta}(t))}{\sum_{k\in S_t} |D_k|}, \\
    \bar {\mathbf{e}}_t &:= \frac{\sum_{k\in S_t} |D_k|\, \mathbf{e}_k(t)}{\sum_{k\in S_t} |D_k|}, \\
    \bar {\mathbf{\zeta}}_t &:= \frac{\sum_{k\in S_t} |D_k|\, {\mathbf{\zeta}}_k(t)}{\sum_{k\in S_t} |D_k|}.
\end{aligned}
\label{eq:averages}
\end{equation}

Note $\nabla F(\boldsymbol{\theta})=\frac{1}{K}\sum_{k=1}^K {\mathbf{g}}_k^\star(\boldsymbol{\theta})$.
\begin{assumption}[Smoothness and Cross-Lipschitz]\label{ass:smooth}
For each $k$, $F_k(\boldsymbol{\theta},{\boldsymbol{\phi}}_k)$ is $L_{\boldsymbol{\theta}}$-smooth in $\boldsymbol{\theta}$, $L_{\boldsymbol{\phi}}$-smooth in ${\boldsymbol{\phi}}_k$, and
\[
\|\nabla_{\boldsymbol{\theta}} F_k(\boldsymbol{\theta},{\boldsymbol{\phi}})-\nabla_{\boldsymbol{\theta} }F_k(\boldsymbol{\theta},{\boldsymbol{\phi}}')\|\le L_{\boldsymbol{\theta}{\boldsymbol{\phi}}}\|{\boldsymbol{\phi}}-{\boldsymbol{\phi}}'\|.
\]
The outer objective $F$ in \eqref{eq:outer-problem} is $L_{\boldsymbol{\theta}}$-smooth.
\end{assumption}

\begin{assumption}[Inner Polyak–Łojasiewicz (PL) \& Stability]\label{ass:inner}
For fixed $\boldsymbol{\theta}$, $F_k(\boldsymbol{\theta},\cdot)$ is $\mu_{\boldsymbol{\phi}}$-PL with $\mu_{\boldsymbol{\phi}}>0$, hence ${\boldsymbol{\phi}}_k^\star(\boldsymbol{\theta})$ is unique and Lipschitz in $\boldsymbol{\theta}$:
\[
\|{\boldsymbol{\phi}}_k^\star(\boldsymbol{\theta}')-{\boldsymbol{\phi}}_k^\star(\boldsymbol{\theta})\|\le C_\star\|\boldsymbol{\theta}'-\boldsymbol{\theta}\|,\qquad C_\star\le L_{\boldsymbol{\theta}{\boldsymbol{\phi}}}/\mu_{\boldsymbol{\phi}}.
\]
\end{assumption}

\begin{assumption}[Noise]\label{ass:noise}
There exist $\sigma_{\boldsymbol{\theta}}^2,\sigma_{\boldsymbol{\phi}}^2<\infty$ such that
$\mathbb{E}[\|{\mathbf{\zeta}}_k(t)\|^2\mid\mathcal{I}_t]\le\sigma_{\boldsymbol{\theta}}^2$ and
$\mathbb{E}[\|\upsilon_k(t)\|^2\mid\mathcal{I}_t]\le\sigma_{\boldsymbol{\phi}}^2$.
\end{assumption}

\begin{assumption}[Top-$\alpha$ Selection Bias/Variance]\label{ass:selection}
Let $\bar {\mathbf{g}}_\star(\boldsymbol{\theta}):=\nabla F(\boldsymbol{\theta})$.
There exist $\beta_{\rm sel}\ge0$ and  sampling variance due to partial device participation $\sigma_{\rm samp}^2<\infty$ such that
\begin{equation}
\begin{aligned}
\big\|\mathbb{E}[\bar {\mathbf{g}}_{S_t,\star}\mid\mathcal{I}_t]-\nabla F(\boldsymbol{\theta}(t))\big\|
&\le \beta_{\rm sel},\\[4pt]
\mathbb{E}\!\left[\big\|\bar {\mathbf{g}}_{S_t,\star}-\mathbb{E}[\bar {\mathbf{g}}_{S_t,\star}\mid\mathcal{I}_t]\big\|^2
\mid\mathcal{I}_t\right]
&\le \frac{\sigma_{\rm samp}^2}{|S_t|},
\end{aligned}
\end{equation}
where $|S_t|=\alpha K$
\end{assumption}

\subsubsection*{Bounding the Personalization-induced Bias}

\begin{lemma}[Inner Tracking Recursion]\label{lem:tracking}
Under Assumptions~\ref{ass:smooth}--\ref{ass:noise}, if $\gamma_t\le 1/L_{\boldsymbol{\phi}}$, then
\begin{equation}
\begin{aligned}
\mathbb{E}\big[\|\Delta_k(t+1)\|\mid\mathcal{I}_t\big]
\!\le& (1\!-\!\gamma_t\mu_{\boldsymbol{\phi}})\|\Delta_k(t)\|\\
\!+&C_\star\,\mathbb{E}\big[\|\boldsymbol{\theta}(t+1)-\boldsymbol{\theta}(t)\|\mid\mathcal{I}_t\big]\\
\!+ &\gamma_t\,\sigma_{\boldsymbol{\phi}}.
\end{aligned}
\label{eq:tracking}
\end{equation}
\end{lemma}

\begin{proof}
From the update \eqref{eq:phi-update}, adding and subtracting ${\boldsymbol{\phi}}_k^\star(\boldsymbol{\theta}(t))$ gives
\begin{align*}
    \Delta_k(t+1)
= &\left[{\boldsymbol{\phi}}_k(t)-\gamma_t\nabla_{{\boldsymbol{\phi}}_k}F_k(\boldsymbol{\theta}(t),{\boldsymbol{\phi}}_k(t))-{\boldsymbol{\phi}}_k^\star(\boldsymbol{\theta}(t))\right] \\
&+ \left[{\boldsymbol{\phi}}_k^\star(\boldsymbol{\theta}(t))-{\boldsymbol{\phi}}_k^\star(\boldsymbol{\theta}(t+1))\right]
- \gamma_t\upsilon_k(t).
\end{align*}
By $\mu_{\boldsymbol{\phi}}$-PL and $\gamma_t\le 1/L_{\boldsymbol{\phi}}$, the first term contracts as
\begin{align*}
\mathbb{E}\left[\|{\boldsymbol{\phi}}_k(t)-\gamma_t\nabla_{{\boldsymbol{\phi}}_k}F_k(\boldsymbol{\theta}(t),{\boldsymbol{\phi}}_k(t))-{\boldsymbol{\phi}}_k^\star(\boldsymbol{\theta}(t))\|\mid\mathcal{I}_t\right] \\
\le (1-\gamma_t\mu_{\boldsymbol{\phi}})\|\Delta_k(t)\|.
\end{align*}
The noise term satisfies $\mathbb{E}[\|\upsilon_k(t)\|\mid\mathcal{I}_t]\le\sigma_{\boldsymbol{\phi}}$.  
By Assumption~\ref{ass:inner},
\[
\|{\boldsymbol{\phi}}_k^\star(\boldsymbol{\theta}(t))-{\boldsymbol{\phi}}_k^\star(\boldsymbol{\theta}(t+1))\|
\le C_\star\|\boldsymbol{\theta}(t+1)-\boldsymbol{\theta}(t)\|.
\]
Combining the above yields \eqref{eq:tracking}.
\end{proof}

\begin{proposition}[Bias from Tracking Error]\label{prop:ebound}
Under Assumptions~\ref{ass:smooth} and \ref{ass:inner},
\begin{align*}
&\|\mathbf{e}_k(t)\|\le L_{\boldsymbol{\theta}{\boldsymbol{\phi}}}\|\Delta_k(t)\|,\\
&\mathbb{E}\|\bar e_t\|^2
\le L_{\boldsymbol{\theta}{\boldsymbol{\phi}}}^2\cdot \frac{1}{|S_t|}\sum_{k\in S_t}\mathbb{E}\|\Delta_k(t)\|^2.
\label{eq:ebound}
\end{align*}
\end{proposition}

\begin{proof}
By cross-Lipschitz in Assumption~\ref{ass:smooth},
$\|\mathbf{e}_k(t)\|=\|\nabla_{\boldsymbol{\theta}} F_k(\boldsymbol{\theta}(t),{\boldsymbol{\phi}}_k(t))-\nabla_{\boldsymbol{\theta}} F_k(\boldsymbol{\theta}(t),{\boldsymbol{\phi}}_k^\star(\boldsymbol{\theta}(t)))\|\le L_{\boldsymbol{\theta}{\boldsymbol{\phi}}}\|\Delta_k(t)\|$.
The bound on $\mathbb{E}\|\bar e_t\|^2$ follows from convexity of $\|\cdot\|^2$ and the definition of $\bar e_t$.
\end{proof}

\subsubsection*{One-Step Descent}

\begin{lemma}[Descent with Explicit Bias]\label{lem:descent}
Under Assumptions~\ref{ass:smooth}--\ref{ass:selection}, with $\boldsymbol{\theta}(t+1)=\boldsymbol{\theta}(t)-\eta_t(\bar {\mathbf{g}}_{S_t,\star}+\bar {\mathbf{e}}_t+\bar{\mathbf{\zeta}}_t)$,
\begin{equation}
\begin{aligned}
\mathbb{E}\!\left[F(\boldsymbol{\theta}(t+1))\mid\mathcal{I}_t\right]
&\le F(\boldsymbol{\theta}(t))-\eta_t\|\nabla F(\boldsymbol{\theta}(t))\|^2\\
&+\eta_t\|\nabla F(\boldsymbol{\theta}(t))\|\big(\beta_{\rm sel}+\|\mathbb{E}[\bar e_t\mid\mathcal{I}_t]\|\big)\\
&+\frac{L_{\boldsymbol{\theta}}\eta_t^2}{2}\left(\mathbb{E}\|\bar {\mathbf{g}}_{S_t,\star}\|^2+\mathbb{E}\|\bar {\mathbf{e}}_t\|^2+\mathbb{E}\|\bar{\mathbf{\zeta}}_t\|^2\right).
\label{eq:descent}
\end{aligned}
\end{equation}
Moreover, $\mathbb{E}\|\bar{\mathbf{\zeta}}_t\|^2\le \sigma_{\boldsymbol{\theta}}^2/\big(w\,|S_t|\big)$ for some $w>0$ depending on the weights $|D_k|$.
\end{lemma}

\begin{proof}
$L_{\boldsymbol{\theta}}$-smoothness gives
$F(\boldsymbol{\theta}(t+1))\le F(\boldsymbol{\theta}(t))+\langle\nabla F(\boldsymbol{\theta}(t)),\boldsymbol{\theta}(t+1)-\boldsymbol{\theta}(t)\rangle+\frac{L_{\boldsymbol{\theta}}}{2}\|\boldsymbol{\theta}(t+1)-\boldsymbol{\theta}(t)\|^2$.
Substitute $\boldsymbol{\theta}(t+1)-\boldsymbol{\theta}(t)=-\eta_t(\bar{ \mathbf{g}}_{S_t,\star}+\bar {\mathbf{e}}_t+\bar{\mathbf{\zeta}}_t)$ and take conditional expectation. Decompose $\mathbb{E}[\bar {\mathbf{g}}_{S_t,\star}\mid\mathcal{I}_t]=\nabla F(\boldsymbol{\theta}(t))+b_t$ with selection bias $\|b_t\|\le\beta_{\rm sel}$ (Assumption~\ref{ass:selection}). Apply Cauchy–Schwarz and use independence/zero-mean of $\bar{\mathbf{\zeta}}_t$ to reach \eqref{eq:descent}. The variance bound follows from averaging over $|S_t|$ with weights bounded away from $0$ and $1$.
\end{proof}

\subsubsection*{Main Convergence Results}

Define constants
\[
C_1:=L_{\boldsymbol{\theta}},\quad
V_{\boldsymbol{\theta}}:=\sigma_{\boldsymbol{\theta}}^2/|S_t|+\sigma_{\rm samp}^2+\beta_{\rm sel}^2.
\]
Let $\bar G:=\sup_t \mathbb{E}\|\bar {\mathbf{g}}_{S_t,\star}\|$. In steady regime, combining Lemma~\ref{lem:tracking} and Proposition~\ref{prop:ebound} under constant stepsizes $\eta_t\equiv\eta\le 1/L_{\boldsymbol{\theta}}$, $\gamma_t\equiv\gamma\le 1/L_{\boldsymbol{\phi}}$ yields
\begin{equation}
\|\mathbb{E}[\bar e_t]\|\le 
L_{\boldsymbol{\theta}{\boldsymbol{\phi}}}\Big(\tfrac{C_\star\eta}{\gamma \mu_{\boldsymbol{\phi}}}\,\bar G+\tfrac{\sigma_{\boldsymbol{\phi}}}{\mu_{\boldsymbol{\phi}}}\Big)
=: \bar E.
\label{eq:Ebar}
\end{equation}

\begin{theorem}[Nonconvex Rate with Explicit Bias]\label{thm:nonconvex}
Under Assumptions~\ref{ass:smooth}--\ref{ass:selection} with $\eta\le 1/(2L_{\boldsymbol{\theta}})$ and $\gamma\le 1/L_{\boldsymbol{\phi}}$,
\begin{equation}
\begin{aligned}
\frac{1}{T}\sum_{t=0}^{T-1}\mathbb{E}\|\nabla F(\boldsymbol{\theta}(t))\|^2
&\le \frac{2\big(F(\boldsymbol{\theta}(0))-F^\star\big)}{\eta T}
+ w_1 C_1\,\eta\,V_{\boldsymbol{\theta}} \\
&\quad + w_2\big(\beta_{\rm sel}+\bar E\big)^2,
\end{aligned}
\label{eq:rate-nonconvex}
\end{equation}

for absolute constants $w_1,w_2>0$, and $\bar E$ given in \eqref{eq:Ebar}.
\end{theorem}

\begin{proof}
Sum \eqref{eq:descent} over $t=0,\dots,T-1$ and telescope. Rearranging gives
\begin{equation}
\begin{split}
\sum_{t=0}^{T-1}\eta\,\mathbb{E}\|\nabla F(\boldsymbol{\theta}(t))\|^2
\le& F(\boldsymbol{\theta}(0))\!-\!F^\star\\
\!+ &\eta\sum_{t=0}^{T-1}\mathbb{E}\|\nabla F(\boldsymbol{\theta}(t))\|
   \big(\beta_{\rm sel}+\|\mathbb{E}[\bar e_t]\|\big) \\
\!\!+ &\frac{L_{\boldsymbol{\theta}}\eta^2}{2}\sum_{t=0}^{T-1}
   \Big(\mathbb{E}\|\bar {\mathbf{g}}_{S_t,\star}\|^2
       + \mathbb{E}\|\bar e_t\|^2
       + \mathbb{E}\|\bar{\mathbf{\zeta}}_t\|^2\Big).
\end{split}
\end{equation}

Use Young’s inequality $ab\le \tfrac12 a^2+\tfrac12 b^2$ on the middle term with $a=\|\nabla F(\boldsymbol{\theta}(t))\|$ and $b=(\beta_{\rm sel}+\|\mathbb{E}[\bar e_t]\|)$, then divide by $T\eta$ and absorb constants into $w_1,w_2$. Bound $\mathbb{E}\|\bar{\mathbf{\zeta}}_t\|^2$ and $\mathbb{E}\|\bar {\mathbf{g}}_{S_t,\star}\|^2$ by $V_{\boldsymbol{\theta}}$ (up to constants). Finally substitute the bound \eqref{eq:Ebar} to obtain \eqref{eq:rate-nonconvex}.
\end{proof}

Theorem~\ref{thm:nonconvex} holds for general smooth nonconvex objectives and serves as our main theoretical guarantee. In the following, we further analyze a special case under the PL condition, which is strictly weaker than strong convexity and does not require convexity, to characterize faster convergence.

\begin{theorem}[PL Outer: Linear Convergence to a Neighborhood]\label{thm:pl}
If $F$ satisfies the PL condition\cite{karimi2016linear} with parameter $\mu_{\boldsymbol{\theta}}>0$, i.e.,
$\tfrac12\|\nabla F(\boldsymbol{\theta})\|^2\ge \mu_{\boldsymbol{\theta}}\big(F(\boldsymbol{\theta})-F^\star\big)$,
then for $\eta\le \min\{1/(2L_{\boldsymbol{\theta}}),\,\mu_{\boldsymbol{\theta}}/L_{\boldsymbol{\theta}}^2\}$,
\begin{equation}
\begin{aligned}
\mathbb{E}\big[F(\boldsymbol{\theta}(t))-F^\star\big]
&\le (1-\eta\mu_{\boldsymbol{\theta}})(t)\big(F(\boldsymbol{\theta}(0))-F^\star\big)\\
&+ \mathcal{O}\!\left(
\frac{L_{\boldsymbol{\theta}}}{\mu_{\boldsymbol{\theta}}}\eta\,V_{\boldsymbol{\theta}}
+ \frac{(\beta_{\rm sel}+\bar E)^2}{\mu_{\boldsymbol{\theta}}}
\right),
\label{eq:pl-rate}
\end{aligned}
\end{equation}
with $\bar E$ as in \eqref{eq:Ebar}.
\end{theorem}

Theorem~\ref{thm:pl} shows that the proposed algorithm converges linearly to a neighborhood of the optimal solution, with the residual error determined by device selection bias and system heterogeneity.

\begin{proof}
Apply PL to \eqref{eq:descent} and take total expectation:
\begin{equation}
\begin{aligned}
\mathbb{E}\left[F(\boldsymbol{\theta}(t+1))\!-\!F^\star\right]
&\le \!\frac{L_{\boldsymbol{\theta}}\eta^2}{2}\,\Xi_t + \mathbb{E}\left[F(\boldsymbol{\theta}(t))\!-\!F^\star\right] \\
&- \!\eta\!\cdot\!2\mu_{\boldsymbol{\theta}}\,\mathbb{E}\left[F(\boldsymbol{\theta}(t))-F^\star\right] \\
& + \!\eta\,\!\sqrt{2\mu_{\boldsymbol{\theta}}}\,
   \mathbb{E}\!\left[\sqrt{F(\boldsymbol{\theta}(t))-F^\star}\,
   (\beta_{\rm sel}+\|\mathbb{E}[\bar e_t]\|)\right] 
\end{aligned}
\end{equation}
where $\Xi_t:=\mathbb{E}\|\bar {\mathbf{g}}_{S_t,\star}\|^2+\mathbb{E}\|\bar e_t\|^2+\mathbb{E}\|\bar{\mathbf{\zeta}}_t\|^2$.
Use $ab\le \frac{\rho}{2}a^2+\frac{1}{2\rho}b^2$ with $a=\sqrt{F(\boldsymbol{\theta}(t))-F^\star}$, $b=\beta_{\rm sel}+\|\mathbb{E}[\bar e_t]\|$ and choose $\rho=\mu_{\boldsymbol{\theta}}$ to absorb the mixed term into $(1-\eta\mu_{\boldsymbol{\theta}})\mathbb{E}[F(\boldsymbol{\theta}(t))-F^\star]$ plus a constant multiple of $(\beta_{\rm sel}+\bar E)^2$. Bound $\Xi_t$ by $V_{\boldsymbol{\theta}}$ (up to constants) and unroll the linear recursion to get \eqref{eq:pl-rate}.
\end{proof}



\ifCLASSOPTIONcaptionsoff
  \newpage
\fi


\bibliography{newRefer}
%

%

\begin{IEEEbiography}[{\includegraphics[width=1in,height=1.25in,clip,keepaspectratio]{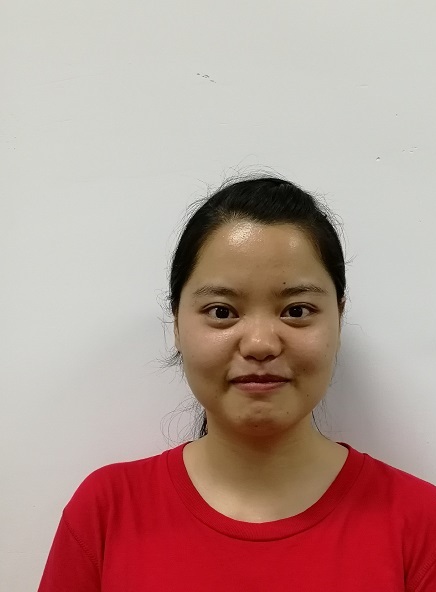}}]{Shiqian Guo}
     received the Ph.D. degree in information and communication engineering from South China University of Technology, in 2024. She is currently a Post-Doctoral Research Fellow with the
    Department of Computer Science, North Carolina State University. Her current research interests include wireless communications, federated learning, and quantum sensing.
\end{IEEEbiography}

\begin{IEEEbiography}
	[{\includegraphics[width=1in,height=1.25in,clip,keepaspectratio]{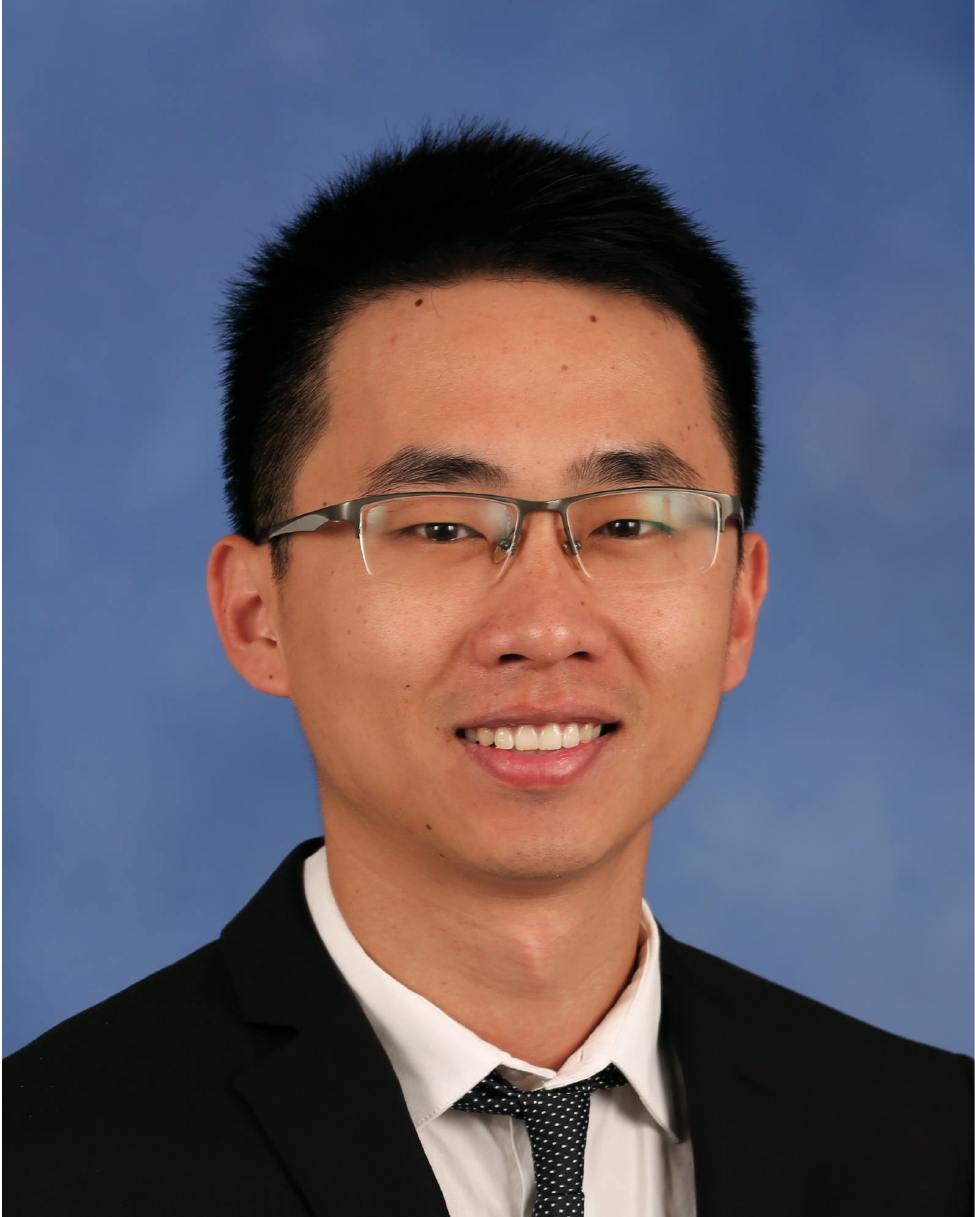}}]
	{Jianqing Liu}(Member, IEEE) is currently an Associate Professor of Computer Science at the NC State University. He received the Ph.D. degree from The University of Florida in 2018 and the B.S. degree from University of Electronic Science and Technology of China in 2013. His research interest is wireless communications and networking, security and privacy. He received the US NSF CAREER Award in 2021. He also received several best paper awards including 2018 Best Journal Paper Award from IEEE Communications Society. He has been serving as the associate editor with \emph{IEEE Transactions on Wireless Communications} and editor with \emph{IEEE Transactions on Communications} since 2024 and was the associate editor with \emph{IEEE Transactions on Vehicular Technology} from 2021 to 2024.
\end{IEEEbiography}

\begin{IEEEbiography}
	[{\includegraphics[width=1in,height=1.25in,clip,keepaspectratio]{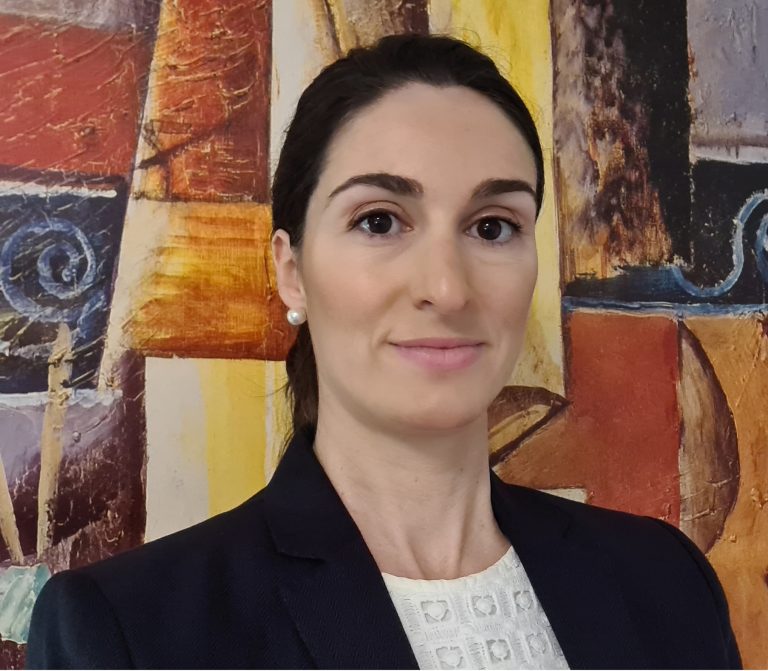}}]
{Beatriz Lorenzo}(Senior Member, IEEE) received the M.Sc. degree in telecommunication engineering from the University of Vigo, Vigo, Spain, in 2008, and the Ph.D. degree from the University of Oulu, Oulu, Finland, in 2012. She is currently an Assistant Professor and the Director of the Network Science Lab, Department of Electrical and Computer Engineering, University of Massachusetts, Amherst, MA, USA. She has published more than 70 papers and coauthored two books on advanced wireless networks. The latest book Artificial Intelligence and Quantum Computing for Advanced Wireless Networks (Wiley, 2022), covers the enabling technologies for the definition, design, and analysis of incoming 6G/7G systems. Her research interests include AI for wireless networks, B5G and 6G network architectures and protocol design, mobile computing, optimization, and network economics. 
    
Dr. Lorenzo was the recipient of the Fulbright Visiting Scholar Fellowship with the University of Florida from 2016 to 2017. She served as Associate Editor of IEEE Transactions on Vehicular Technology and IEEE Transactions on Mobile Computing. She was the General Co-Chair for WiMob Conference in 2019 and serves regularly in the TPC of top IEEE and ACM conferences.

\end{IEEEbiography}






\end{document}